\documentclass{article}

\usepackage{microtype}
\usepackage{graphicx}
\usepackage{subfigure}
\usepackage{booktabs} % for professional tables
\usepackage{algorithm}
\usepackage{algorithmic}
\usepackage{bbm}

\usepackage{hyperref}

\usepackage[accepted]{icml2025}

\usepackage{amsmath}
\usepackage{amssymb}
\usepackage{mathtools}
\usepackage{amsthm}

\usepackage[capitalize,noabbrev]{cleveref}

\usepackage[textsize=tiny]{todonotes}

\icmltitlerunning{Conditional Distribution Quantification}

\newtheorem{theorem}{Theorem}[section]
\newtheorem*{theorem*}{Theorem}
\newtheorem{proposition}{Proposition}[section]
\theoremstyle{definition}

\theoremstyle{remark}

\newcommand{\field}[1]{\mathbb{#1}}
\newcommand{\R}{\field{R}}
\newcommand{\E}{\field{E}}
\renewcommand{\P}{\field{P}}
\newcommand{\Loi}{{\mathcal L}}
\newcommand{\W}{{\mathcal W}}
\newcommand{\abs}[1]{\lvert#1\rvert}
\newcommand{\bigabs}[1]{\bigl\lvert#1\bigr\rvert}

\newcommand{\bigpar}[1]{\bigl(#1\bigr)}
\newcommand{\Bigpar}[1]{\Bigl(#1\Bigr)}

\newcommand{\Bigcro}[1]{\Bigl[#1\Bigr]}

\begin{document}

\twocolumn[
    \icmltitle{Conditional Distribution Quantization in Machine Learning}


    \icmlsetsymbol{equal}{*}

    \begin{icmlauthorlist}
    \icmlauthor{Blaise Delattre}{equal,fox,dauph}
    \icmlauthor{Sylvain Delattre}{equal,cite}
    \icmlauthor{Alexandre Vérine}{equal,ens}
    \icmlauthor{Alexandre Allauzen}{dauph,psl}
\end{icmlauthorlist}

\icmlaffiliation{fox}{FOXSTREAM, Vaulx-en-Velin, France}
\icmlaffiliation{dauph}{Miles Team, LAMSADE, Universit\'e Paris-Dauphine, PSL University, Paris, France}
\icmlaffiliation{cite}{Universit\'e Paris Cit\'e and Sorbonne Universit\'e, CNRS, Laboratoire de Probabilit\'es, Statistique et Mod\'elisation, Paris, France}
\icmlaffiliation{ens}{\'Ecole Normale Sup\'erieure, Université PSL, DIENS, Paris, France
 }
\icmlaffiliation{psl}{ESPCI PSL, Paris, France}

\icmlcorrespondingauthor{Blaise Delattre}{bldelattre@gmail.com}

    \icmlkeywords{Machine Learning, ICML}

    \vskip 0.3in
]



\printAffiliationsAndNotice{\icmlEqualContribution} % otherwise use the standard text.

\begin{abstract}
    Conditional expectation \( \mathbb{E}(Y \mid X) \) often fails to capture the complexity of multimodal conditional distributions \( \mathcal{L}(Y \mid X) \). To address this, we propose using \( n \)-point conditional quantizations—functional mappings of \( X \) that are learnable via gradient descent—to approximate \( \mathcal{L}(Y \mid X) \).
    This approach adapts Competitive Learning Vector Quantization (CLVQ), tailored for conditional distributions.
    It goes beyond single-valued predictions by providing multiple representative points that better reflect multimodal structures. It enables the approximation of the true conditional law in the Wasserstein distance. The resulting framework is theoretically grounded and useful for uncertainty quantification and multimodal data generation tasks. For example, in computer vision inpainting tasks, multiple plausible reconstructions may exist for the same partially observed input image \( X \). We demonstrate the effectiveness of our approach through experiments on synthetic and real-world datasets.
\end{abstract}

\section{Introduction}

When analyzing the relationship between random variables \( X \) and \( Y \), the conditional expectation \( \mathbb{E}(Y \mid X) \) is often used as a summary statistic to describe \( Y \) given \( X \).
However, for multimodal or complex conditional distributions, \( \mathbb{E}(Y \mid X) \) may fail to capture critical information about the underlying distribution. For instance, consider a case where:
\(
\P(Y = X + 100 \mid X) = \frac{1}{2}, \quad \P(Y = X - 100 \mid X) = \frac{1}{2},
\)
then:
\(
\E(Y \mid X) = X, \quad \text{but} \quad \Loi(Y \mid X) = \frac{1}{2}\delta_{X+100} + \frac{1}{2}\delta_{X-100}.
\)
where $\delta$ stands for Dirac mass.
While the expectation \( \mathbb{E}(Y \mid X) \) returns \( X \), the true conditional law \( \Loi(Y \mid X) \) reveals a richer, multimodal structure that is entirely missed by relying solely on the mean.

This limitation has important implications for downstream tasks, particularly in uncertainty quantification for image restoration models used in safety-critical domains such as autonomous driving and biological imaging. Many existing approaches rely on per-pixel estimates, such as variance heatmaps~\citep{kendall2017uncertainties} or confidence intervals~\citep{angelopoulos2022image} to visualize uncertainty. 
While these methods provide valuable insights, they can struggle to represent structured uncertainty, overlooking spatial correlations between neighboring pixels.

\citet{bishop1994mixture}'s work underscores this limitation, showing that conditional expectation works well when \( Y \) lies in a corner of the simplex but often fails in other contexts. Instead, Bishop proposes a Mixture Density Networks, which model the full conditional distribution, allowing for a richer representation of \( Y \) given \( X \).
More recent works utilize PCA on the conditional distribution to model uncertainty in image reconstruction~\citep{nehme2023uncertainty, manor2024posterior}. These approaches focus on identifying the dominant modes of variation (principal components) to capture uncertainty. However, they are local methods, analyzing variability only around the conditional mean.

To address this limitation, we propose approximating \( \Loi(Y \mid X) \) using \( n \)-point quantizations that adapt to \( X \).
Our approach, termed \textit{Conditional Competitive Learning Vector Quantization} (CCLVQ), extends \textit{Competitive Learning Vector Quantization}~\citep{PagesBouton, kohonen1995clvq} to conditional distributions.
In this framework, the \( n \)-point representations are functional mappings $\{f_i\}_{i=1}^n$ of \( X \), enabling flexible and expressive summaries of \( Y \) while capturing multimodal structures. In particular, the $f_i$ can be parameterized as deep neural networks.
Our contributions are:
\begin{itemize}
    \item Theoretical Foundation of CCLVQ: We extend CLVQ to conditional distributions and provide theoretical guarantees for quantizing \( \Loi(Y \mid X) \) with adaptive representations.

    \item Algorithm Design: We propose an efficient training algorithm for CCLVQ that learns an ensemble of experts to collaboratively partition the output space.

    \item Uncertainty Quantification in Image Restoration and Generation: We apply CCLVQ to uncertainty quantification tasks, such as inpainting and denoising, capturing the multimodal nature of \( \Loi(Y \mid X) \) and providing diverse, plausible reconstructions with quantitative uncertainty.  We also apply the method to improve quality and diversity in generative models.
\end{itemize}
\section{Related Works}
Understanding and modeling uncertainty in deep learning has been a long-standing challenge, with various approaches proposed to quantify and manage uncertainty across tasks.

\textbf{Pixel-based approaches}
Pixel-wise uncertainty quantification methods, such as per-pixel variance heatmaps~\citep{kendall2017uncertainties}, confidence intervals~\citep{angelopoulos2022image}, and conformal prediction masks~\citep{kutiel2023conformal}, often neglect spatial correlations between pixels, resulting in overconfident predictions and limited diversity in reconstructions. Many of these methods rely on Bayesian approaches~\cite {gal2016dropout, blundell2015weight}, which, while theoretically sound, are computationally intensive and difficult to scale. Methods that explore the conditional distribution by producing diverse reconstructions while accounting for pixel correlations are essential.

\textbf{PCA-based approaches}
Building on the idea of exploring conditional distributions, Principal Component Analysis (PCA) provides a form of uncertainty quantization~\citep{nehme2023uncertainty, manor2024posterior} that focuses on variability around the conditional mean \( \mathbb{E}(Y \mid X) \). A similar method has been used to regularize the top principal components of the conditional covariance for enhanced posterior sampling in generation with GANs~\citep{bendel2024pcagan}.
However, PCA-based methods face critical limitations: they capture only linear dependencies and local variability, neglecting global or non-linear structures.
This makes them inherently unsuited for modeling multimodal distributions. Furthermore, PCA assumes unimodal variability, quantifying uncertainty in \( \ell_2 \)-norm around the mean, which restricts their ability to provide diverse reconstructions.

\textbf{Ensemble of Experts}
Deep Ensembles~\citep{lakshminarayanan2017simple} provide a scalable alternative by training multiple independent networks, combining their outputs to capture both epistemic and aleatoric uncertainty, demonstrating robustness to dataset shifts and well-calibrated predictions.
Mixture Density Networks (MDN)~\citep{bishop1994mixture} predict the parameters of a mixture distribution, enabling them to capture multimodal outputs and quantify predictive uncertainty effectively. This makes them particularly suitable for tasks where the conditional distribution \( \mathcal{L}(Y \mid X) \) is complex or multimodal.
MDNs approximate complex distributions by predicting the parameters of probabilistic mixtures, while quantization methods achieve a similar goal by representing the distribution with a finite set of adaptive points.

\textbf{Multiple Choice Learning} MCL was introduced by \citet{guzman2012mcl} to address prediction under ambiguity by training $n$ predictors jointly with a winner-takes-all (WTA) assignment strategy. Subsequent extensions have aimed at improving stability and diversity: stochastic MCL (sMCL) \citep{lee2016stochasticmcl} performs batch-wise assignments, confident MCL (cMCL) \cite{lee2017confidentmcl} adds an entropic regularization to encourage uniform usage of predictors, and \citet{rupprecht2017learning} provides a probabilistic interpretation of MCL through mixture modeling. More recently, resilient MCL (rMCL) \citep{letzelter2023resilient} introduces a learned scoring scheme for hypothesis selection, underpinned by a Voronoi-based probabilistic framework to preserve output diversity. Annealed MCL \cite{perera2024annealed} further extends MCL by introducing an annealing schedule that gradually relaxes WTA constraints during training, thus improving optimization stability. 
More recently, the work of \citet{wta2024letzelter} explores the probabilistic interpretation of Winner-Takes-All (WTA) training, highlighting its connection to conditional density estimation, which is close to the work in this paper.

In parallel, implicit generative models offer alternative strategies for modeling multimodality. Implicit Maximum Likelihood Estimation (IMLE) \citep{li2019implicit} avoids explicit likelihoods by matching real samples to generated outputs, ensuring distributional coverage and preventing mode collapse. CHIMLE \citep{peng2022chimle} extends this idea to conditional generation through a hierarchical formulation, synthesizing diverse outputs in complex multimodal tasks. These methods provide complementary approaches to ambiguity modeling: MCL emphasizes deterministic expert selection, whereas IMLE-based techniques leverage stochasticity to ensure coverage of the output space.

\textbf{Quantization methods}
Lloyd's K-means algorithm~\citep{lloyd1982least} is a foundational quantization method that iteratively minimizes the distortion between data points and their nearest centroids by alternating between assignment and update steps. While effective, it requires exhaustive traversal of the dataset in each iteration, making it computationally expensive for large-scale data. Extensions of K-means, such as Competitive Learning Vector Quantization (CLVQ)~\citep{kohonen1995clvq, PagesBouton}, focus on adaptively learning quantization points through a competitive process.

In the following, we propose a method based on CLVQ, resembling an ensemble of experts, providing a recipe for training diverse networks to quantize simply the conditional distribution.

\section{Background}
Here, we present some well-known results, including the proofs showing the extent to which the quantization of a random vector provides an approximation of its distribution for the Wasserstein distance with a distribution with finite support.
Refer to \cite{Graf&Luschgy} or \cite{procPages} for more details.

\subsection{Wasserstein distance}
The Wasserstein distance of order $2$ between two distributions $\mu$ and $\nu$ on $\R^d$ is defined
as
$$ \W_2(\mu,\nu) = \inf_{Y_1\sim\mu, Y_2\sim\nu} \Bigpar{ \E\Bigpar{\abs{Y_1-Y_2}^2 } }^{1/2}$$
where the infimum is taken over all pairs of $d$-dimensional random vectors $Y_1$ and $Y_2$ marginally distributed as $\mu$ and $\nu$, respectively. Here, $\abs{\cdot}$ denotes the Euclidean norm on $\R^d$.
As long as $\mu$ and $\nu$ have finite second moments, $\W_2(\mu,\nu)$ is finite.

The Wasserstein distance is a powerful metric between probability measures. Unlike measures such as total variation or Kullback-Leibler divergence, it respects the geometry of the underlying space \(\mathbb{R}^d\), making it particularly meaningful in applications requiring geometric fidelity~\citep{peyre2019computational}.

\subsection{Vector Quantization to approximate a distribution}
Let $Y$ be a random vector with values in $\R^d$ and let $n \ge 1$, an integer.
We denote by $|\cdot|$ the Euclidean norm on $\R^d$.
Vector quantization, including methods like \(k\)-means and Learning Vector Quantization (LVQ), aims to minimize the functional
\[
    D_n(\alpha) = \mathbb{E}\Bigl[\min_{1 \leq i \leq n} \bigabs{Y - \alpha_i}^2\Bigr],
\]
where \(\alpha = (\alpha_1, \dots, \alpha_n) \in (\mathbb{R}^d)^n\) represents the quantizer.
We assume that $\E(\abs{Y}^2) <\infty$ so that $D_{n}(\alpha)<\infty$
for every $\alpha$.
This functional arises naturally in applications such as signal transmission, clustering, and numerical integration~\citep{procPages}. For a given quantizer \(\alpha \in (\mathbb{R}^d)^n\),
one chooses a closest neighbor projection $\pi_\alpha:\R^d\to\{\alpha_1,\dots,\alpha_n\}$
that is a map such that, for all $y\in\R^d$,
\begin{equation} \label{projection}
    \pi_\alpha(y) \in \{\alpha_1, \dots, \alpha_n\} \ \text{and} \
    |y - \pi_\alpha(y)| = \min_{1 \leq i \leq n} |y - \alpha_i |
\end{equation}
This choice specifies how ties are broken. The induced $\alpha$-quantization of $Y$ is the random vector $\pi_\alpha(Y)$ and we have
%The quantized version of \(Y\), denoted \(\pi_\alpha(Y)\), has distortion
\[
    D_n(\alpha) = \mathbb{E}\Bigl[\bigabs{Y - \pi_\alpha(Y)}^2\Bigr] .
\]
Consequently, the Wasserstein distance of order 2 between the distribution of $Y$ and $\pi_\alpha(Y)$ satisfies
\begin{equation} \label{majW}
    \mathcal{W}_2\bigl(\mathcal{L}(Y), \mathcal{L}(\pi_\alpha(Y))\bigr) \leq D_n(\alpha)^{1/2}.
\end{equation}
Moreover, \(\mathcal{L}(\pi_\alpha(Y))\) can be expressed as a weighted sum:
\[
    \mathcal{L}(\pi_\alpha(Y)) = \sum_{i=1}^n \mathbb{P}(\pi_\alpha(Y) = \alpha_i) \delta_{\alpha_i},
\]
allowing for practical approximations of \(\mathcal{L}(Y)\) through estimated weights \(\mathbb{P}(\pi_\alpha(Y) = \alpha_i)\).

We have more than the upper bound \eqref{majW}: $D_n(\alpha)^{1/2}$ is the Wasserstein distance of $\Loi(Y)$ to $\Loi(\pi_\alpha(Y))$ and is also the Wasserstein distance of $\Loi(Y)$ to the set of distributions with support in $\{\alpha_1,\dots,\alpha_n\}$.
\begin{proposition} \label{simple1}
    For every $\alpha \in(\R^d)^n$ we have
    \begin{multline*} \mathcal{W}_2\bigl(\mathcal{L}(Y), \mathcal{L}(\pi_\alpha(Y))\bigr) = D_n(\alpha)^{1/2}\\
        = \min_{\nu\in\mathcal{P}(\alpha_1,\dots,\alpha_n)} \W_2\bigpar{\Loi(Y),\nu}
    \end{multline*}
    where $\mathcal{P}(\alpha_1,\dots,\alpha_n)$ denotes the set of distributions with support in the set $\{\alpha_1,\dots,\alpha_n\}$.
\end{proposition}
\begin{proof}
    In view of the upper bound \eqref{majW}, it suffices to prove that $\W_2\bigpar{\Loi(Y),\nu}\ge D_n(\alpha)^{1/2}$ for every $\nu\in\mathcal{P}(\alpha_1,\dots,\alpha_n)$. Let $\nu\in\mathcal{P}(\alpha_1,\dots,\alpha_n)$.
    If $\tilde Y$ and $\hat Y$ are random vectors with values in $\R^d$ such that
    $\Loi(\tilde Y)=\Loi(Y)$ and $\Loi(\hat Y)=\nu$, then $\P(\hat Y\in\{\alpha_1,\dots,\alpha_n\})=1$ which implies
    $$ \E\bigpar{\abs{\tilde Y-\hat Y}^2} \ge \E\bigpar{\min_{1\le i \le n} \abs{\tilde Y-\alpha_i}^2}
        =D_{n}(\alpha)$$
\end{proof}
\subsection{About optimal quantizers}
There exists at least one optimal quantizer, that is $\alpha^*\in(\R^d)^n$ such that
$$ D_n(\alpha^*)= \inf_{\alpha\in(\R^d)^n} D_n(\alpha), $$
see \cite{procPages} for a proof.
The following result shows why it is worth minimizing $D_n$:
\begin{proposition} \label{simple2} If $\alpha^*$ is optimal, one has
    \begin{multline*} \mathcal{W}_2\bigpar{\Loi(Y), \Loi(\pi_{\alpha^*}(Y))}=D_{n}(\alpha^*)^{1/2}\\
        = \min_{\nu\in\mathcal{P}_n(\R^d)} \W_2\bigpar{\Loi(Y),\nu}
    \end{multline*}
    where $\mathcal{P}_n(\R^d)$ denotes the set of distribution on $\R^d$  whose support has a cardinality smaller than $n$.
\end{proposition}
\begin{proof}
    From Proposition \ref{simple1} one gets
    $$D_n(\alpha^*)^{1/2}=\min_{\alpha\in(\R^d)^n} \min_{\nu\in\mathcal{P}(\alpha_1,\dots,\alpha_n)} \W_2\bigpar{\Loi(Y),\nu},$$
    and moreover, since $\bigcup_{\alpha} \mathcal{P}(\alpha_1,\dots,\alpha_n)=\mathcal{P}_n(\R^d)$,
    one has
    \begin{multline*}
        \min_{\alpha\in(\R^d)^n} \min_{\nu\in\mathcal{P}(\alpha_1,\dots,\alpha_n)} \W_2\bigpar{\Loi(Y),\nu}
        = \\ \min_{\nu\in\mathcal{P}_n(\R^d)} \W_2\bigpar{\Loi(Y),\nu}.
    \end{multline*}
    %TO DO  It's Lemma 3.4 of \cite{Graf&Luschgy} quand $\alpha=\alpha^*$
\end{proof}

\subsection{Competitive Learning Vector Quantization}
A way to minimize $D_n(\alpha)$ and to obtain the corresponding minimizer $\alpha$ is to use Competitive Learning Vector Quantization (CLVQ).
In this section, we briefly present the CLVQ algorithm and explain on which condition it is similar to gradient descent.

Let $\mathcal{D} = \{y_j\}_{j=1}^N$ be a dataset.
For all $\alpha\in\R^d$ and $y\in\R^d$, one chooses $I_\alpha(y)\in\{1,\dots,n\}$ such that
$$\abs{y-\alpha_{I_\alpha(y)} } = \min_{1\le i\le n}\abs{y-\alpha_i}$$

To minimize the distortion, CLVQ  consists of the recursive scheme where one iteration is written as:

$$ \alpha \longleftarrow \alpha - \gamma\Bigpar{1_{\{I_\alpha(Y)= i\}} (\alpha_i-Y)}_{1\le i\le n}$$

where $Y$ is sampled from $\mathcal{D}$ and $\gamma\in(0,1]$ is the learning rate, the winning $\alpha_i$ is getting closer to $Y$.
This is almost a stochastic gradient descent:
\begin{proposition}[\cite{procPages}]
    For all $\alpha\in(\R^d)^n$ such that
    \begin{equation}
        \label{eq:hypothesis}
        \P\Bigpar{\exists\ i\not=j : \abs{Y-\alpha_i}=\abs{Y-\alpha_j}} = 0,
    \end{equation}
    $D_n$ is differentiable at $\alpha$ and
    $$\nabla D_n(\alpha) = 2\E\Bigcro{\Bigpar{1_{\{I_\alpha(Y)=i\}} (\alpha_i-Y)}_{1\le i\le n}}$$
\end{proposition}
Note that for optimal $\alpha^\star$ Assumption~\ref{eq:hypothesis} is always true, see \cite{Graf&Luschgy}.

While vector quantization provides a robust framework for approximating static distributions, many real-world scenarios involve conditional dependencies. For instance, the distribution of \(Y\) may vary with an auxiliary random variable \(X\). To address this, we extend vector quantization to \textit{conditional quantization}, where the quantizers adapt to \(X\) as learnable functions, allowing the approximation of the conditional distribution of $Y$ given $X$.% \(Q(x, dy) = \mathbb{P}(Y \in dy \mid X = x)\).

\section{Conditional Quantization}
Let $X$ and $Y$ be two random vectors, defined on the same probability space, where $X$ takes its values in an arbitrary space $E$ and $Y$ takes its values in $\R^d$.
The objective is to construct an approximation of the conditional law \(Q(x, dy) = \mathbb{P}(Y \in dy \mid X = x)\), which describes the distribution of \(Y\) given \(X = x\). More precisely, $Q(x,\cdot)$, $x\in E$,
is family of distribution on $\R^d$ such that for all test function $F:\R^d\to\R$ one has, with probability $1$:
$$ \E\bigpar{F(Y) \mid X} = \int_{\R^d} F(y) Q(X,dy) .$$

For a function \(f = (f_1, \dots, f_n): E \to (\mathbb{R}^d)^n\), define:
\[
    \Delta_n(f) = \mathbb{E}\Bigl[\min_{1 \leq i \leq n} \bigl|Y - f_i(X)\bigr|^2\Bigr].
\]

The framework of conditional quantization focuses on minimizing \(\Delta_n(f)\) to approximate \(Q(X, \cdot)\) in the Wasserstein sense. This approach generalizes the concept of vector quantization by making the quantizers dependent on \(X\), enabling flexible modeling of conditional distributions. We describe this formalization and its practical applications below.

\textbf{Integrability assumption}
We assume that there exists $f:E\to(\R^d)^n$ such that $\Delta_n(f)<\infty$. This implies that $E\bigpar{\abs{Y}^2 \mid X} < \infty$ with probability $1$ since
$$\abs{Y}\le \min_{1\le i\le n}\abs{Y-f_i(X)} + \max_{1\le i \le n} \abs{f_i(X}.$$
In other words, with probability $1$:
$$\int_{\R^d} \abs{y}^2 Q(X,dy) <\infty .$$

The distortion functional \(\Delta_n(f)\) can also be expressed as:
\[
    \Delta_n(f) = \mathbb{E}\Bigl[\bigl|Y - \pi_{f(X)}(Y)\bigr|^2\Bigr],
\]
where \(\pi_{f(X)}(Y)\) is a closest neighbor projection defined by \eqref{projection},
and the conditional law of \(\pi_{f(X)}(Y)\) given \(X = x\) can be written as:
\[
    \widehat{Q}_f(x, \cdot) = \sum_{i=1}^n Q\bigl(x, \{y \in \mathbb{R}^d : \pi_{f(x)}(y) = f_i(x)\}\bigr)\,\delta_{f_i(x)}(\cdot).
\]

\textbf{Connection to the Wasserstein distance}
Applying Proposition \ref{simple1} ``conditionally to $X$", one obtains:
\begin{theorem}
    \begin{enumerate}
        \item
              With probability $1$, one has
              \begin{align*}
                   & \mathbb{E}\Bigl[\min_{1 \leq i \leq n} \bigl|Y - f_i(X)\bigr|^2 \mid X\Bigr]                     \\
                   & = \int_{\mathbb{R}^d} \min_{1 \leq i \leq n} \bigl|y - f_i(X)\bigr|^2 Q(X, dy)                   \\
                   & = \mathcal{W}_2\bigl(Q(X, \cdot), \widehat{Q}_f(X, \cdot)\bigr)^2                                \\
                   & = \min_{\nu \in \mathcal{P}(f_1(X), \dots, f_n(X))} \mathcal{W}_2\bigl(Q(X, \cdot), \nu\bigr)^2,
              \end{align*}
              where \(\mathcal{P}(f_1(X), \dots, f_n(X))\) denotes the set of distributions with support in the set \(\{f_1(X), \dots, f_n(X)\}\).
        \item On has
              %By integrating with respect to \(X\), we obtain:
              \[
                  \Delta_n(f) = \mathbb{E}\Bigl[\mathcal{W}_2\bigl(Q(X, \cdot), \widehat{Q}_f(X, \cdot)\bigr)^2\Bigr].
              \]
    \end{enumerate}
\end{theorem}
\begin{proof}
    \begin{enumerate}
        \item The first equality follows from the definition of $Q$.
              Applying Proposition \ref{simple1} to a random vector with distribution $Q(x,\cdot)$, where $x\in E$ is such that
              $\int \abs{y}^2 Q(x,dy)<\infty$,
              yields
              \begin{align*}
                   & \int_{\mathbb{R}^d} \min_{1 \leq i \leq n} \bigl|y - f_i(x)\bigr|^2 Q(x, dy)                     \\
                   & = \mathcal{W}_2\bigl(Q(x, \cdot), \widehat{Q}_f(x, \cdot)\bigr)^2                                \\
                   & = \min_{\nu \in \mathcal{P}(f_1(x), \dots, f_n(x))} \mathcal{W}_2\bigl(Q(x, \cdot), \nu\bigr)^2.
              \end{align*}
              Since $\int \abs{y}^2 Q(X,dy) <\infty$ with probability $1$, one gets the second and the third equalities.
        \item It's a consequence of the first point and the fact that
              $$\Delta_n(f)=\E\Bigpar{ \mathbb{E}\Bigl[\min_{1 \leq i \leq n} \bigl|Y - f_i(X)\bigr|^2 \mid X\Bigr] }$$
    \end{enumerate}
\end{proof}
Thus, minimizing \(\Delta_n(f)\) is equivalent to reducing the expected squared Wasserstein distance between the true conditional law \(Q(X, \cdot)\) and its quantized approximation \(\widehat{Q}_f(X, \cdot)\).

\textbf{Optimal conditional quantizer}
Here we show there always exists an optimal quantizer for $\Delta_n$ and we establish how good is $\widehat Q_{f^*}(X, \cdot)$ as an approximation of $Q(X, \cdot)$ in  Wasserstein distance.
\begin{theorem}
    There exists a quantizer $f^*:E \to (\R^d)^n$ such that
    $$ \Delta_n(f^*)=\inf_f \Delta_n(f). $$
    Moreover, for any optimal quantizer $f^*$, with probability $1$, one has
    \[
        \mathcal{W}_2\bigl(Q(X, \cdot), \widehat Q_{f^*}(X, \cdot)\bigr) = \min_{\nu \in \mathcal{P}_n(\mathbb{R}^d)} \mathcal{W}_2\bigl(Q(X, \cdot), \nu\bigr),
    \]
    where \(\mathcal{P}_n(\mathbb{R}^d)\) denotes the set of distributions on \(\mathbb{R}^d\) with support of cardinality at most \(n\).
\end{theorem}
The proof is presented in the Appendix~\ref{appendix:proofs}.

Conditional quantization extends the concept of vector quantization to approximate the conditional distribution of \(Y\) given \(X\). By minimizing \(\Delta_n(f)\), one obtains a quantized approximation \(\hat{Q}_f(X, \cdot)\) that is close to \(Q(X, \cdot)\) in the expected Wasserstein sense.
This process involves not only learning the quantizer but also estimating the weights
$$ x\mapsto Q\bigl(x, \{y \in \mathbb{R}^d : \pi_{f(x)}(y) = f_i(x)\}\bigr) $$
associated with it to accurately represent the conditional distribution.
In the next subsection, we present how to learn such quantizers, including the estimation of their associated weights.

\begin{algorithm}[t]
    \caption{One Batch Iteration of CCLVQ}
    \label{alg:distortion_with_classifier}
    \begin{algorithmic}[1]
        \REQUIRE
        \STATE $\mathcal{B} = \{x_j, y_j\}_{j=1}^N$: batch of samples
        \STATE $n$: number of parametric functions (experts)x
        \STATE $\{f_i(x; \theta_i)\}_{i=1}^n$: functions with parameters $\theta_i$
        \STATE $\ell$ : loss function from $\R^d \times \R^d \to \R$
        \STATE $h(x; \phi)$: classifier with parameters $\phi$
        \STATE $\gamma_{\mathrm{exp}}, \gamma_{\mathrm{cls}}$: learning rates for experts and classifier
        % \STATE $T$: number of epochs

        \STATE \textbf{Initialize} $\theta_1, \ldots, \theta_n$ and $\phi$
        % \FOR{$t = 1$ to $T$}
        \FOR{$j=1$ to $N$}
        % \STATE Generate distorted data $y_j \gets g(x_j)$ using transformation $g$
        \STATE $i^*(j) \gets \arg\min_{i \in \{1, \dots, n\}} \ell\bigpar{y_j, f_i(x_j; \theta_i)}$
        \COMMENT{assign sample to its closest function (expert)}
        \ENDFOR

        \FOR{$i = 1$ to $n$}
        \STATE $\displaystyle \theta_i \gets \theta_i - \gamma_{\mathrm{exp}} \,\nabla_{\theta_i}\!\Bigl(\sum_{j: i^*(j) = i} \ell(y_j, f_i(x_j; \theta_i)) \Bigr)$
        \COMMENT{update only function $i$ on samples assigned to $i$}
        \ENDFOR

        \STATE \COMMENT{train the classifier on the expert assignments}
        \STATE $\displaystyle \ell_{\text{classifier}} \gets -\frac{1}{N} \sum_{j=1}^N \sum_{i=1}^n \mathbbm{1}_{[i^*(j) = i]} \log h_i(x_j; \phi)$
        \COMMENT{cross-entropy loss with expert labels}
        \STATE $\displaystyle \phi \gets \phi - \gamma_{\mathrm{cls}} \nabla_\phi \mathcal{L}_{\text{classifier}}$
        % \ENDFOR
        \STATE \textbf{Return} $\{\theta_i\}_{i=1}^n$, $\phi$
    \end{algorithmic}
\end{algorithm}
\subsection{Conditional Competitive Learning Vector Quantisation (CCLVQ)}
We introduce the Conditional Competitive Learning Vector Quantisation (CCLVQ) algorithm to minimize the distortion function \(\Delta_{n}(f)\), where the \(n\) functions \(f = (f_1, \dots, f_n)\) are parameterized by \(\theta  =(\theta_1, \dots, \theta_n)\), respectively.

CCLVQ is similar to CLVQ and, thus is closely linked to gradient descent. One iteration associated with one data point $(X, Y)$ is written as
\begin{align*}
    \theta \gets \theta -\gamma\Bigpar{1_{\{I_{f(X)}(Y)= i\}} \nabla_{\theta_i} \bigl| f_i(X, \theta_i) - Y  \bigr|^2}_{1\le i\le n}
\end{align*}
We generalize it to any loss function $\ell:\R^d\times\R^d\to\R$:
\begin{align*}
    \theta \gets \theta -\gamma\Bigpar{1_{\{I_{f(X)}(Y)= i\}} \nabla_{\theta_i}\ell(Y, f_i(X, \theta_i))  }_{1\le i\le n} \ .
\end{align*}
where, for $\alpha\in(\R^d)^n$ and $y$ in $\R^d$, $I_\alpha(y)\in\{1,\dots,n\}$ satisfies
$$\ell(y,\alpha_{I_\alpha(y)} ) = \min_{1\le i\le n}\ell(y,\alpha_i).$$
The algorithm selectively updates only the parameters of the function that minimizes the distortion for a given sample, iteratively refining the functions to better represent the conditional distribution \(\mathcal{L}(Y \mid X)\). In this context, the functions \(f_i\) are referred to as experts.
Each iteration alternates between two key steps:
(1) assigning each sample to its closest expert \(f_i\) based on the current parameters \(\theta_i\), and
(2) updating the parameters of the assigned expert using gradient descent. In parallel, a classifier \(h(x; \phi)\) is trained to predict the probabilities associated with each expert. These probabilities, or weights, are crucial for approximating the conditional distribution \(\mathcal{L}(Y \mid X)\), as they capture the relative contribution of each expert while also providing a measure of uncertainty, which can be crucial in reconstruction settings such as medical imaging.

\textbf{Introducing new experts}
We employ the splitting strategy~\citep{linde1980vector} to introduce a new expert during training. This approach involves duplicating an existing expert, typically the one with the highest contribution to the overall distortion \(\Delta_n(f)\) or the largest number of assigned samples. Adding a new expert through splitting is performed during the final training epoch to minimize computational overhead.

\begin{figure}
    \centering
    \includegraphics[width=1\linewidth]{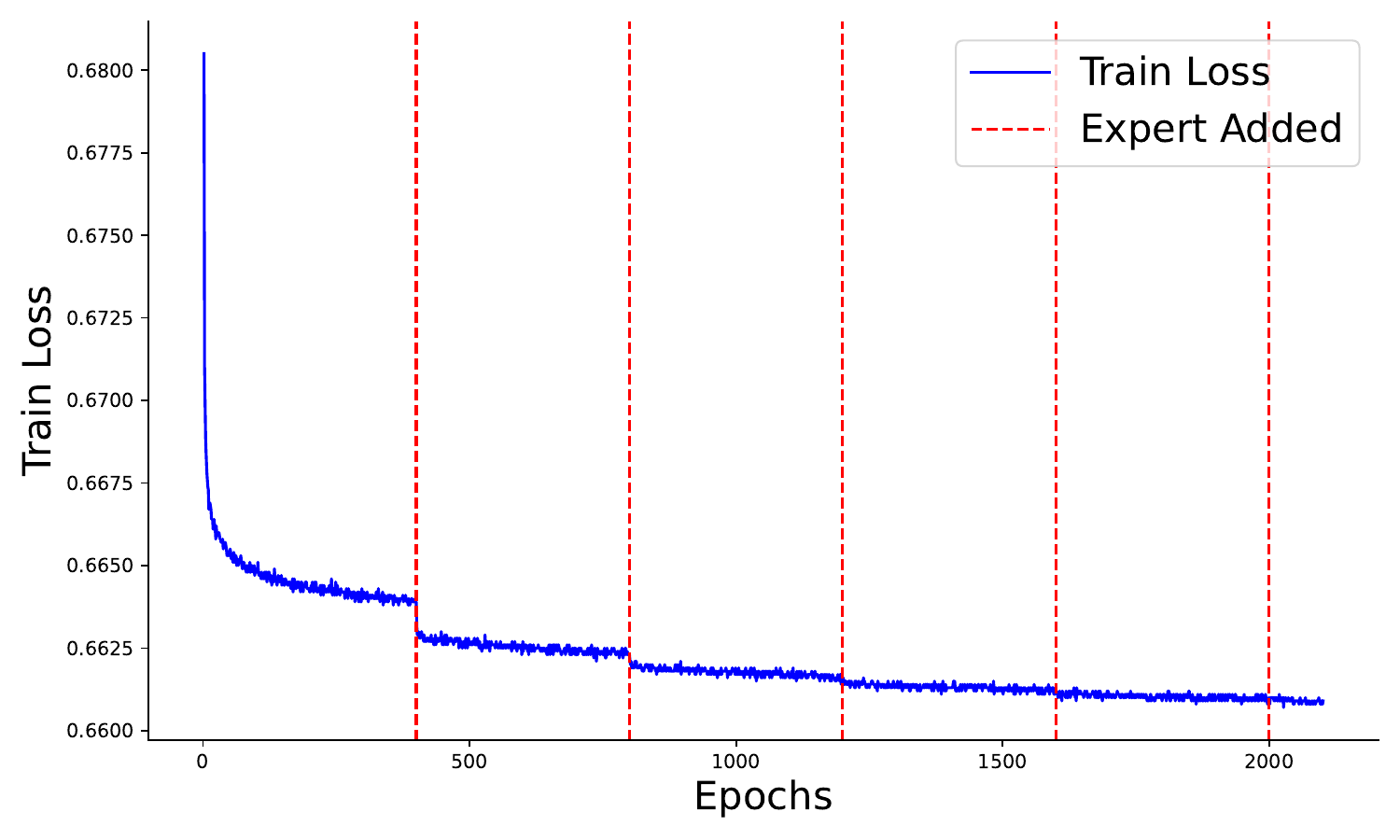}
    \caption{CCLVQ training loss on the MNIST dataset for an inpainting task, showing the effect of adding new experts using the splitting strategy every 400 epochs.
    }
    \label{fig:training_loss}
\end{figure}

\textbf{Choice of $\mathbf{n}$}
To illustrate the impact on the distortion functional \(\Delta_n(f)\) of adding a new expert, we experiment with an inpainting task with a random window on MNIST, see Appendix~\ref{appendix:traning_detail_mnist} for more training details. As shown in Figure~\ref{fig:training_loss}, the training loss decreases each time a new expert is added.
To isolate the impact of adding an expert, we deliberately allowed the model to train for many epochs before introducing a new expert, ensuring convergence before the addition.
The choice of the number of experts \( n \) is empirical and depends on the task. In this case, we observe a significant improvement in performance for \( n = 3 \), beyond which the improvements become less pronounced.
A practical approach to reducing computational overhead on large models and datasets is, to begin with, a pretrained model, performing splitting from this pretrained state, and then fine-tuning the resulting experts for a few additional epochs.

\begin{figure}[h]
    \centering
    \includegraphics[width=\linewidth]{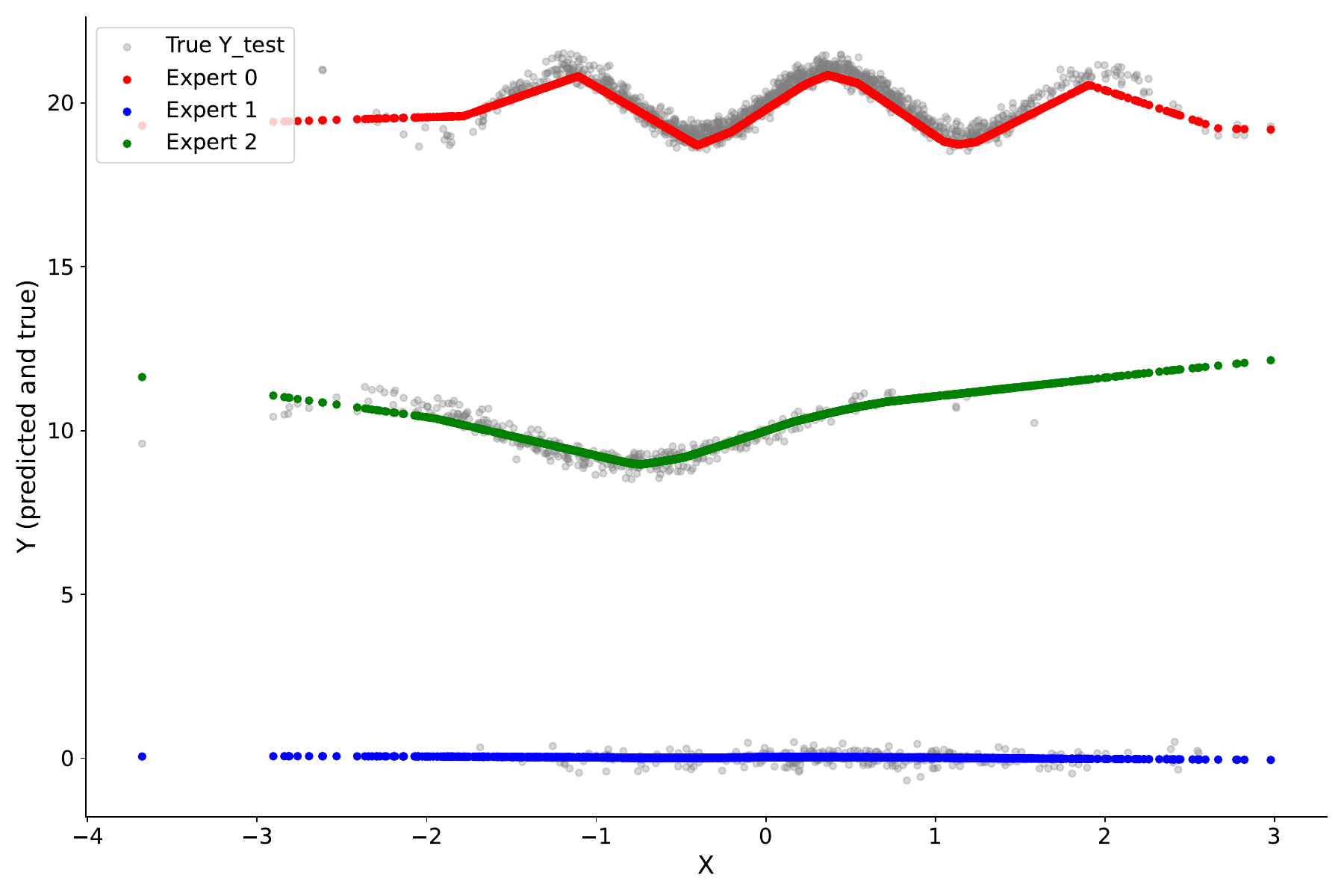}
    \caption{Visualization of synthetic data showing true data points $Y$ (in grey) and predicted values from every expert (in colors). Each expert captures a specific mode of conditional distribution.}
    \label{fig:synthetic-multimodal}
\end{figure}
% \textcolor{blue}{swap X and Y}
\section{Examples of applications}
\subsection{Multi-value Regression on synthetic data}
We aim to demonstrate CCLVQ's applications through multi-value regression, multi-modal reconstruction, and generation experiments. Upon acceptance, the code will be made publicly available. We illustrate Algorithm~\ref{alg:distortion_with_classifier} in a multimodal regression problem. Each expert is a small MLP with one hidden layer of dimension $20$ trained with gradient descent following CCLVQ.
The data generation process is defined as follows:
\begin{align*}
     & X \sim \mathcal{N}(0, 1), \quad
    p_i(x) = \frac{\exp(a_i x + b_i)}{\sum_{j=1}^{\mathrm{mode}} \exp(a_j x + b_j)},            \\
     & \text{knowing $X$, $I \sim \mathrm{Cat}\bigl(p_1(X), \dots, p_{\mathrm{mode}}(X)\bigr)$} \\
     & \quad \text{and $\varepsilon \sim \mathcal{N}(0, \sigma^2)$ are independent,}            \\
     & Y = \sin\bigl(2I X\bigr) + 10I + \varepsilon.
\end{align*}
Figure~\ref{fig:synthetic-multimodal} shows that each expert learns one mode of the conditional distribution \(\mathcal{L}(Y \mid X)\). In the figure, there are $3$ experts and $3$ modes.
We can also train a classifier $h$ to learn the vector $p(X)$ to have the probability associated with each expert, as depicted in Figure~\ref{fig:real_proba_vs_predicted}. Experts’ weights better approximate mode probabilities in regions with higher sample density.
\begin{figure}
    \centering
    \includegraphics[width=1\linewidth]{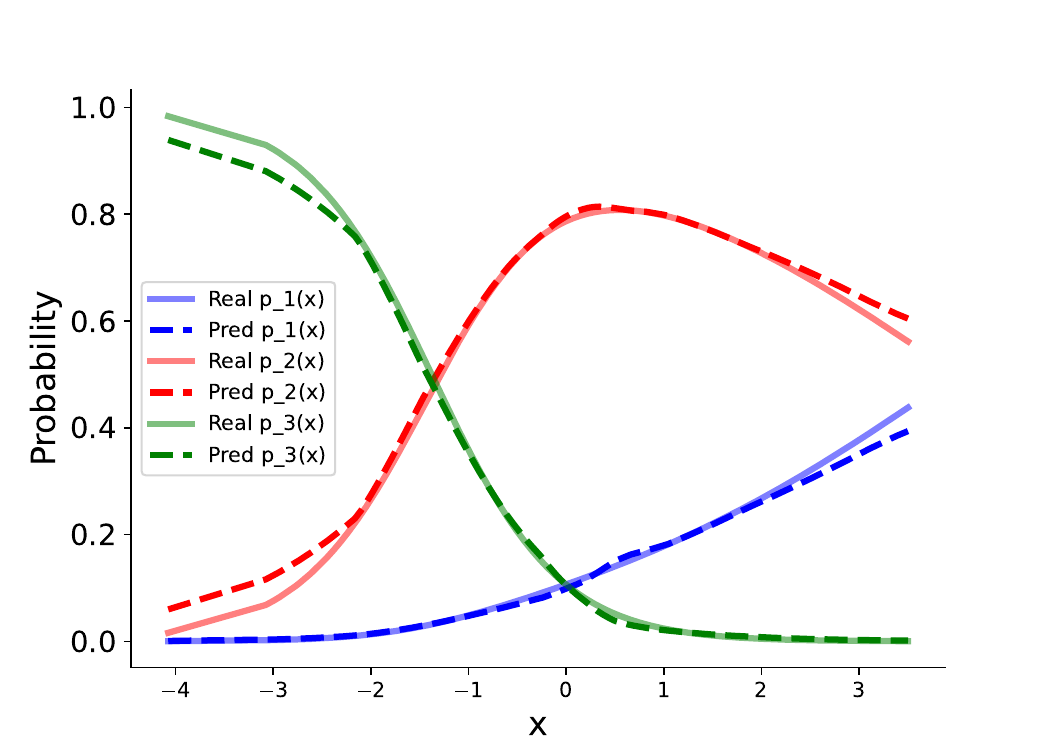}
    \caption{Weights estimations of experts versus probabilities of mode w.r.t $X$.}
    \label{fig:real_proba_vs_predicted}
\end{figure}
\begin{figure}[t]
    \centering
    \includegraphics[width=\linewidth,trim={0 1 0 0},clip]{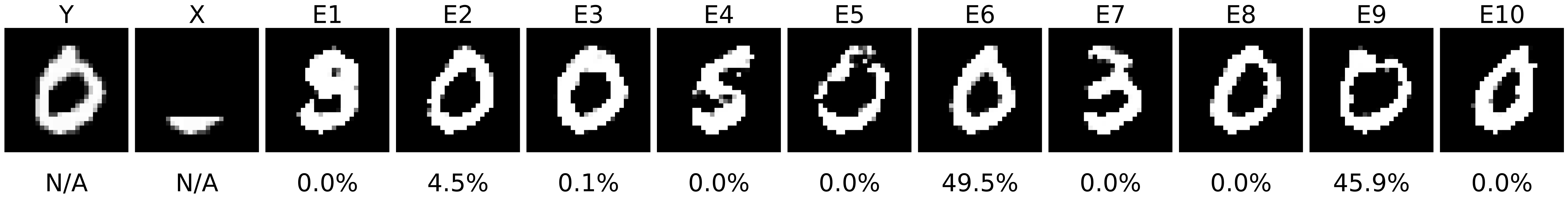}
    \includegraphics[width=\linewidth,trim={0 1 0cm 3},clip]{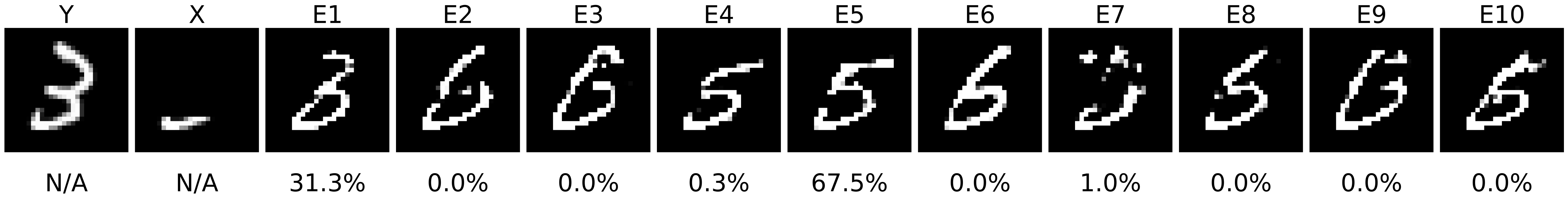}
    \includegraphics[width=\linewidth,trim={0 1 0cm 3},clip]{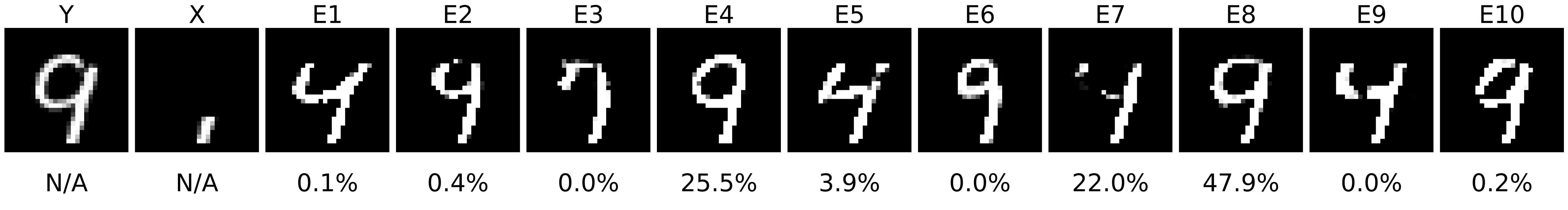}
    \caption{Multi-modal inpainting results on MNIST. Each row shows the ground truth $Y$, masked input $X$, and reconstructions from different experts, followed by the classifier's weights as a measure of uncertainty.}
    \label{fig:multi-modal-inpainting}
\end{figure}
\label{experiment:inpainting_mnist}
\begin{figure}[t]
    \centering
    \includegraphics[width=\linewidth,trim={0 1 0 0},clip]{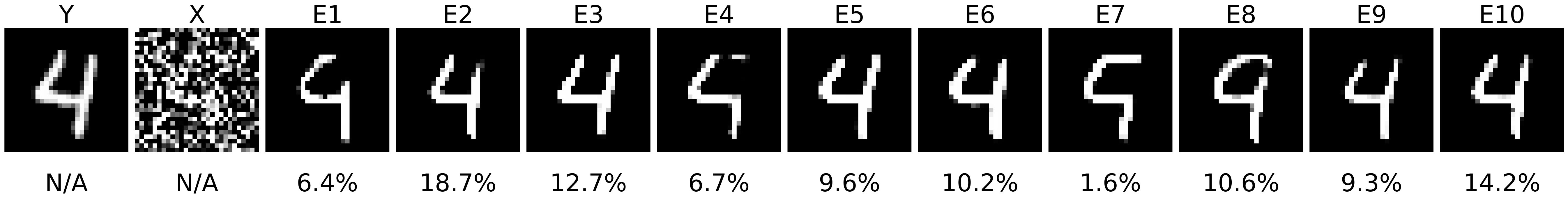}
    \includegraphics[width=\linewidth,trim={0 1 0cm 8},clip]{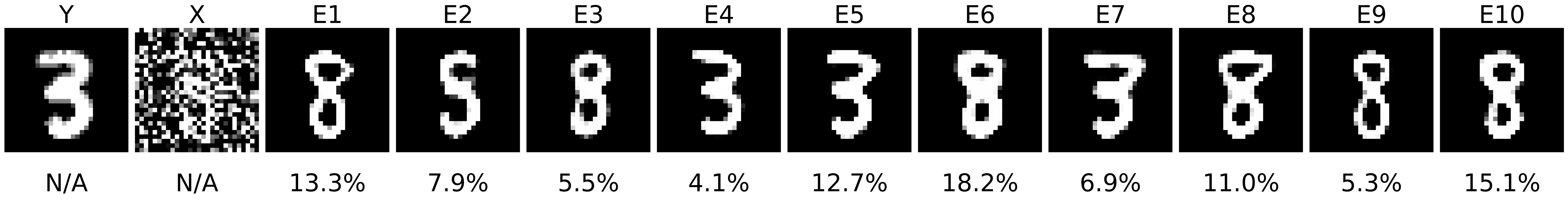}
    \includegraphics[width=\linewidth,trim={0 1 0cm 8},clip]{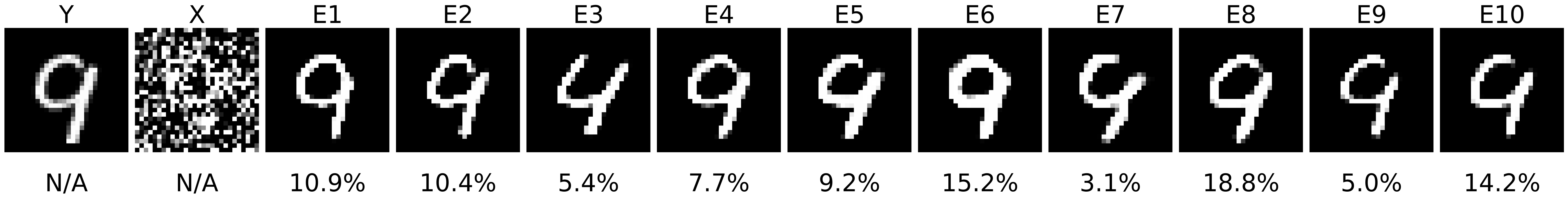}
    \caption{Multi-modal denoising results on MNIST. Each row shows the ground truth $Y$, noised input $X$, and reconstructions from different experts, followed by the classifier's weights as a measure of uncertainty.}
    \label{fig:multi-modal-denoising}
\end{figure}

\subsection{Multi-modal Reconstruction with Uncertainty Quantification}
\paragraph{MNIST Dataset}
\label{experiment:inpainting_mnist}
In this experiment, we perform inpainting on MNIST digits, where the upper part of each image is masked, as in~\cite{nehme2023uncertainty}. The goal is to reconstruct the masked region while considering the inherent ambiguity and multimodality of possible completions, we use an encoder-decoder architecture similar to the one from~\citet{pathak2016context} for each expert and a ResNet9 for the classifier $h$.
Figure~\ref{fig:multi-modal-inpainting} illustrates the inpainting process on a test set. Each column corresponds to a reconstruction generated by a different expert model $f_i$. The first column $X$ represents the base image, serving as the ground-truth target for reconstruction. The second column $Y$ displays the masked image, which is the input provided to the models. Subsequent columns depict reconstructions produced by the various expert models $f_i$, capturing different plausible completions of the masked region.

The weights of each expert learned by the classifier $h$ during training are visualized in the last row. These weights represent the probability distribution over experts for a given reconstruction and provide a measure of uncertainty by indicating the most probable reconstruction. 
% Note that the work of~\cite{nehme2023uncertainty, manor2024posterior} could not quantify uncertainty in such a manner.
In the same manner, we also illustrate reconstruction for denoising tasks, see Figure~\ref{fig:multi-modal-denoising}.
Reconstruction for a unique model trained is provided for the comparison of the same samples with Figure~\ref{fig:one_expert_inpainting} and Figure~\ref{fig:one_expert_denoising} in the Appendix.

\textbf{CelebA-HQ Dataset}
\begin{figure}[t]
    \centering
    \includegraphics[width=1\linewidth]{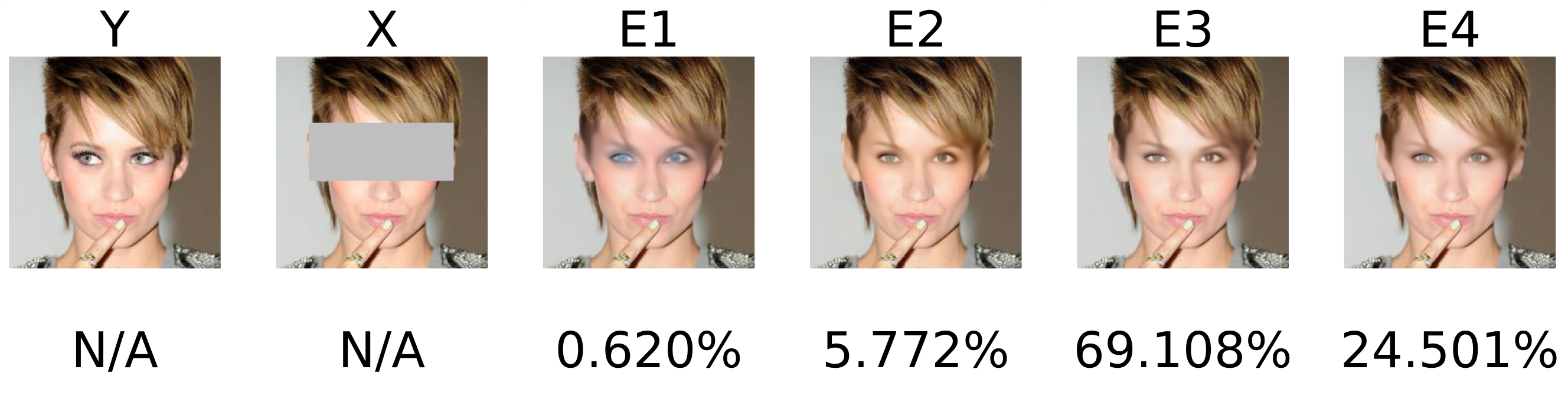}
    \centering
    \includegraphics[width=1\linewidth]{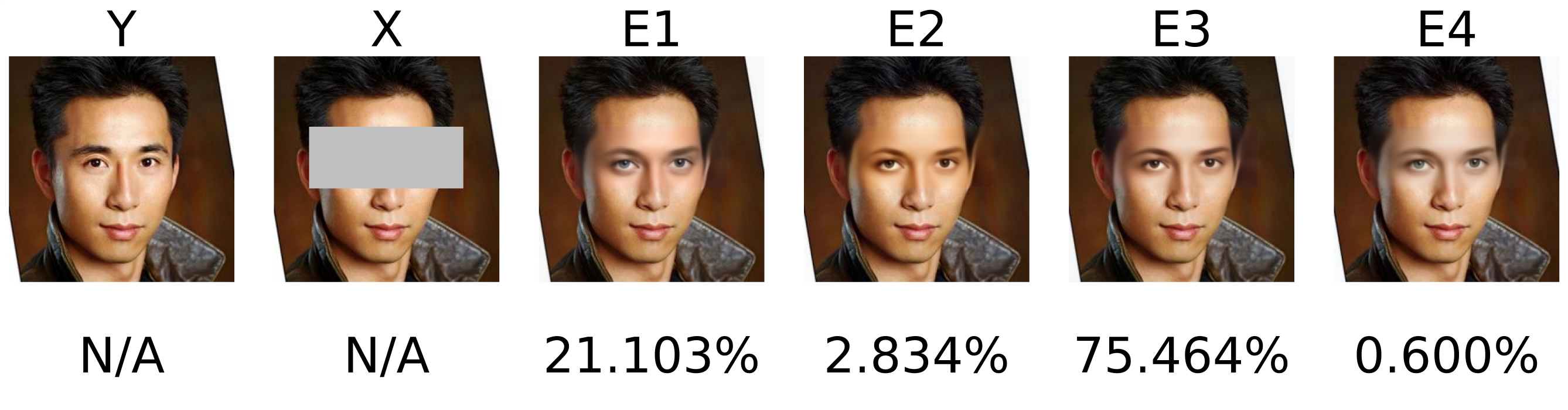}
    \centering
    \includegraphics[width=1\linewidth]{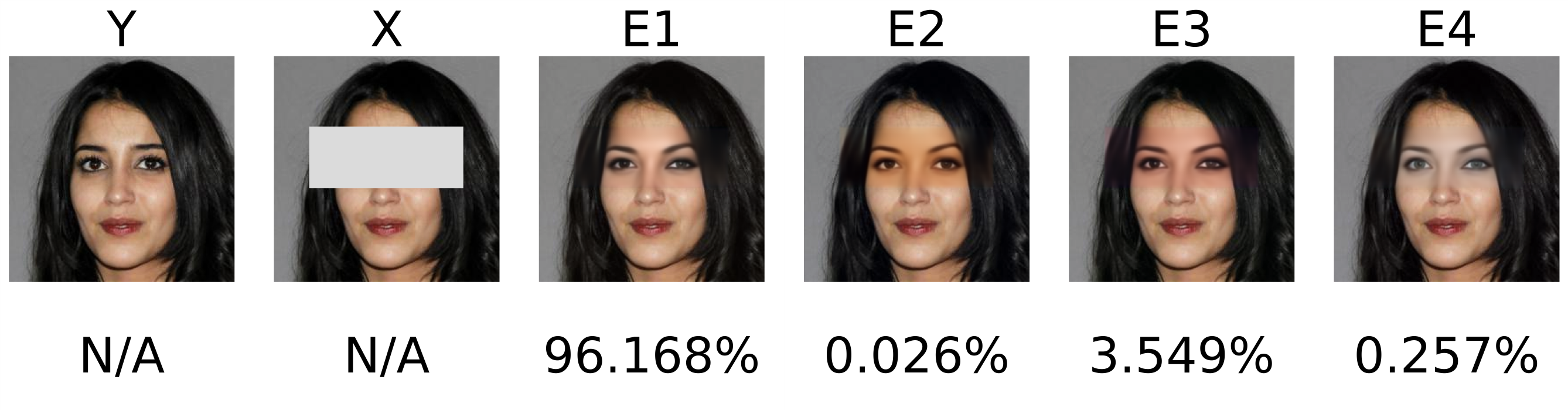}
    \caption{Diverse inpainting results on the CelebA-HQ dataset using CCLVQ. Each column shows reconstructions from different expert models after masking the area near the eyes, highlighting variations in eye color, eyebrow shape, and other fine-grained features. Probability for each construction are placed below image.
    }
    \label{fig:celb_inpainting}
\end{figure}
We conducted an inpainting experiment on the CelebA-HQ dataset, masking the area near the eyes on the test set, following the protocol of~\cite{nehme2023uncertainty}. The same ResUNet architecture~\cite{ronneberger2015unet, falk2019unet} was used for this task.
The baseline model was trained for 30 epochs with a batch size of $16$. Subsequently, we split the model into four experts and applied CCLVQ for only one additional epoch, significantly reducing computational overhead.
This approach led to diverse reconstructions, see Figure~\ref{fig:celb_inpainting} and more reconstructions in the Appendix see %Figure~\ref{fig:more_faces} and 
Figure~\ref{fig:more_celeb_inpainting}, demonstrates variations in eye color, eyebrow shape, and other fine-grained features, illustrating the method's ability to capture multimodal distributions effectively.

\begin{figure*}[t!]
    \subfigure[Dataset]{\includegraphics[width=0.16\linewidth]{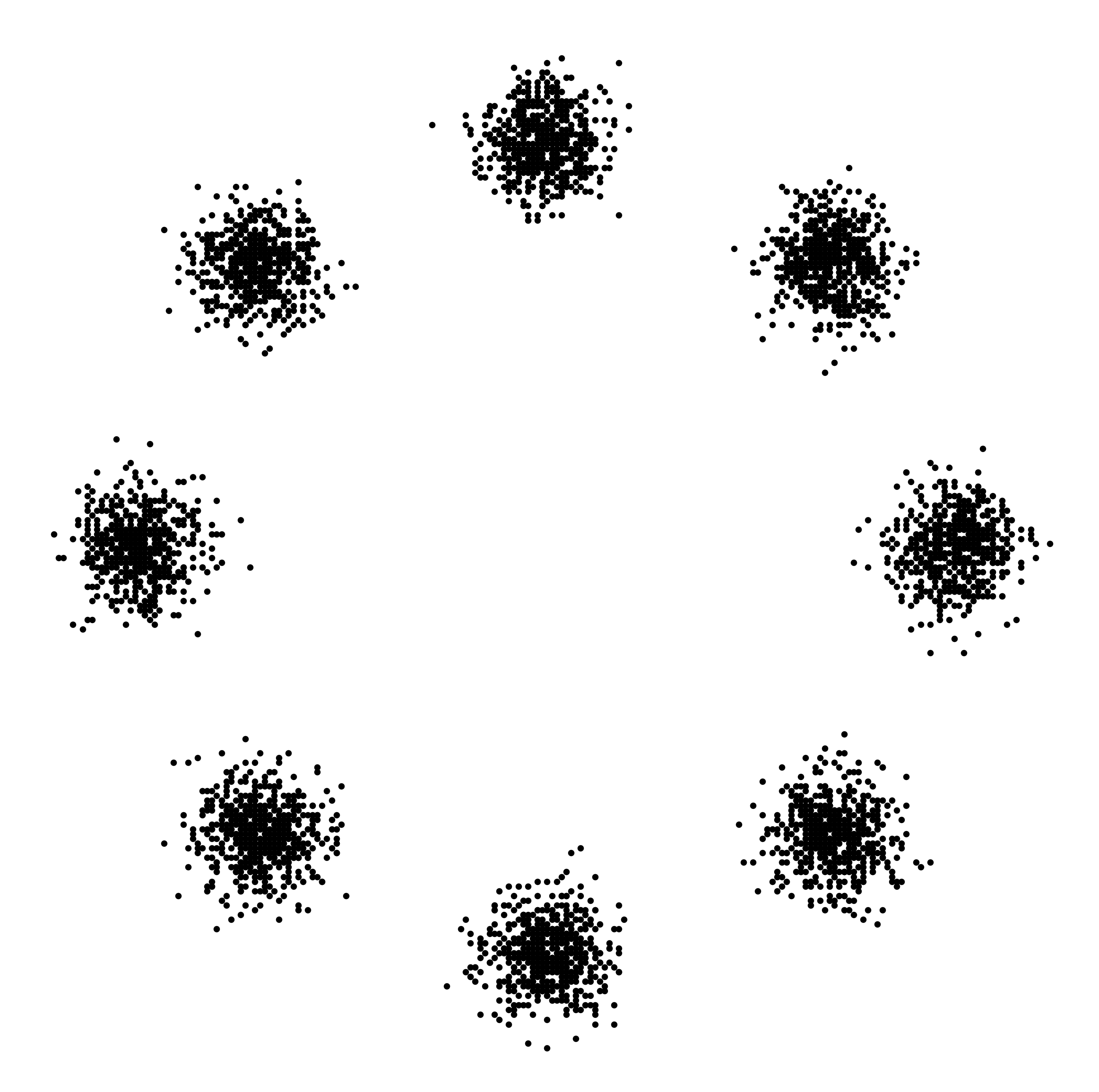}\label{fig:nfdata}} \vline
    \subfigure[$n=1$]{\includegraphics[width=0.16\linewidth]{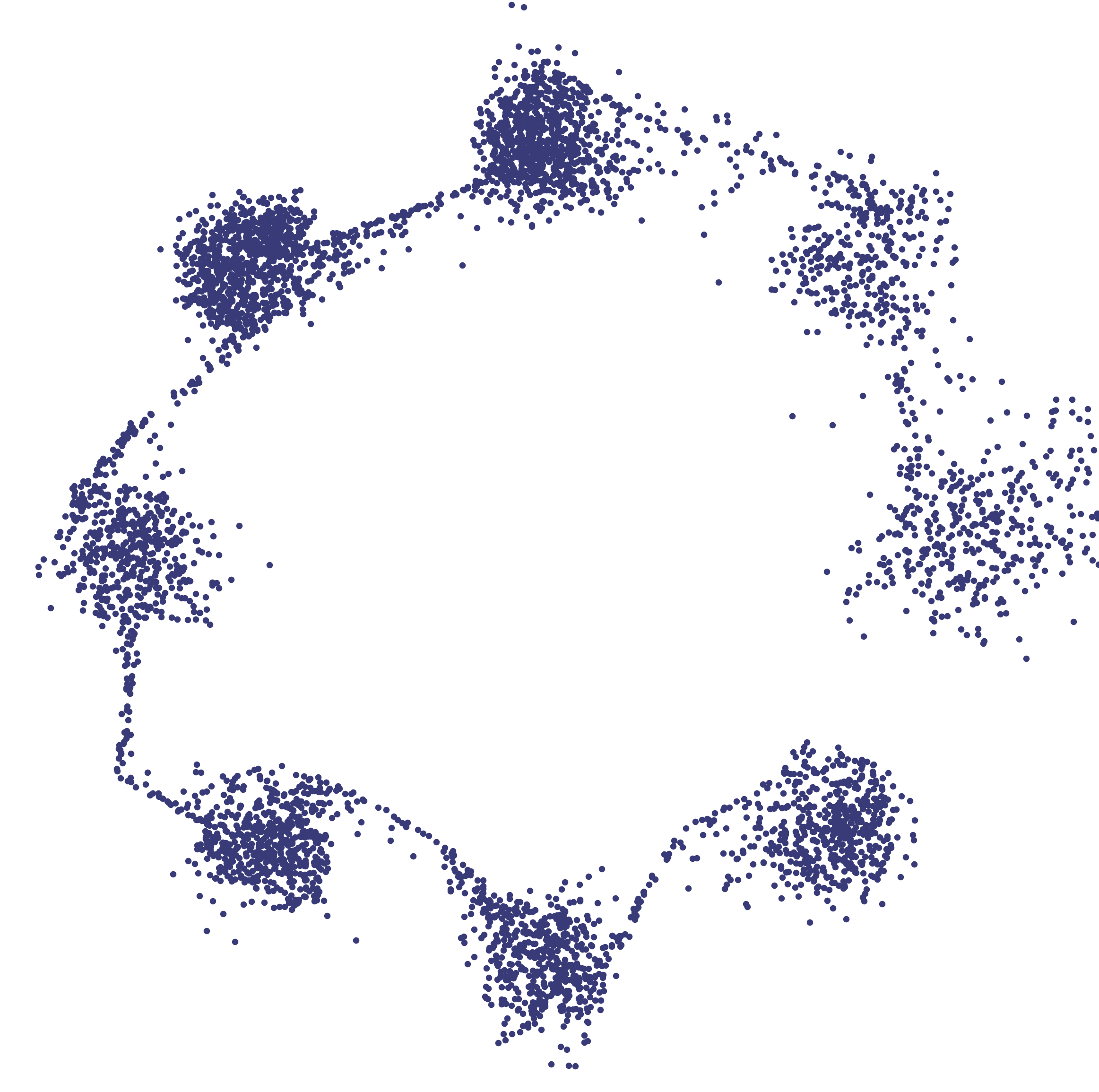}\label{fig:nf1}}
    \subfigure[$n=2$]{\includegraphics[width=0.16\linewidth]{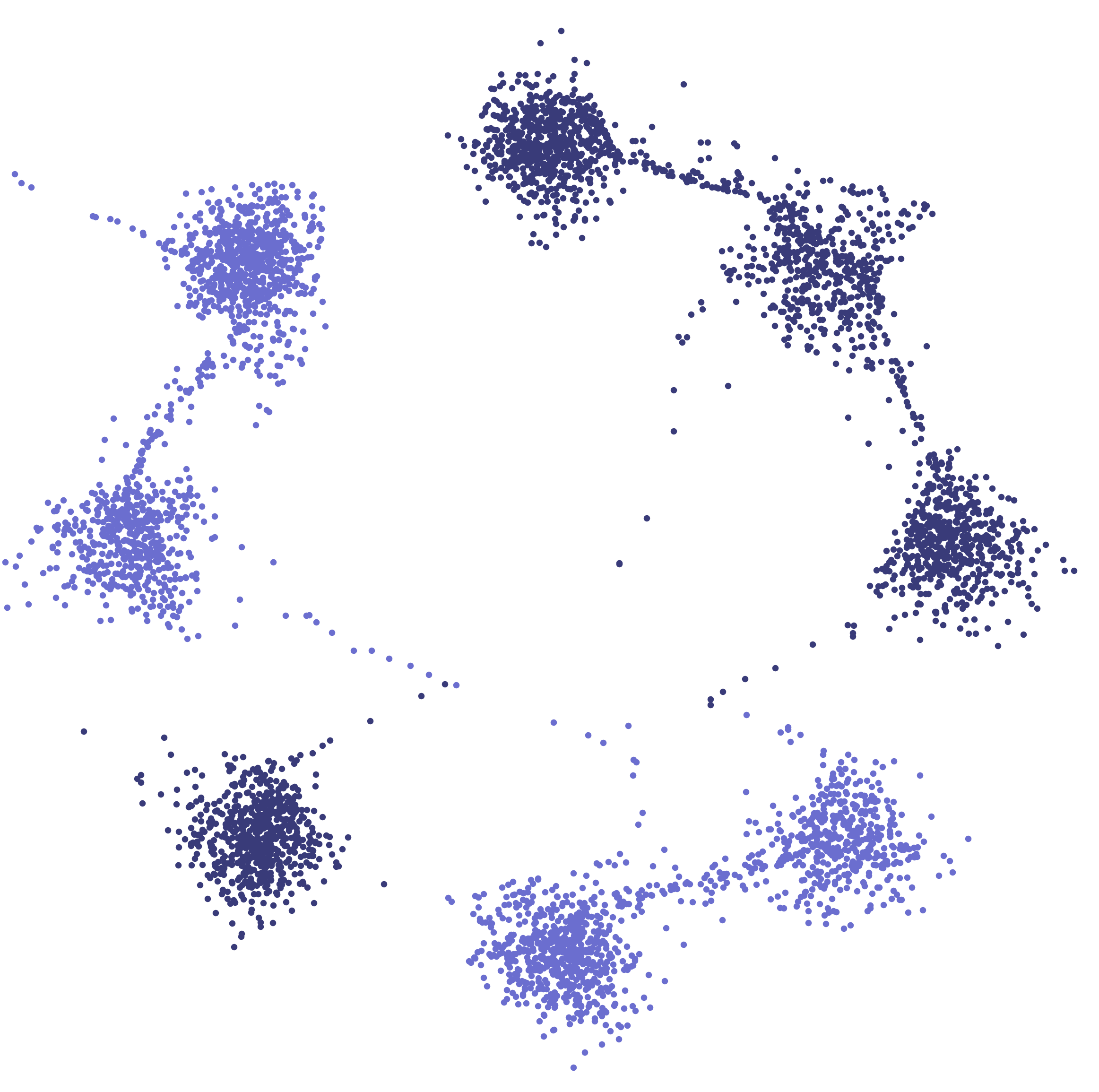}}
    \subfigure[$n=4$]{\includegraphics[width=0.16\linewidth]{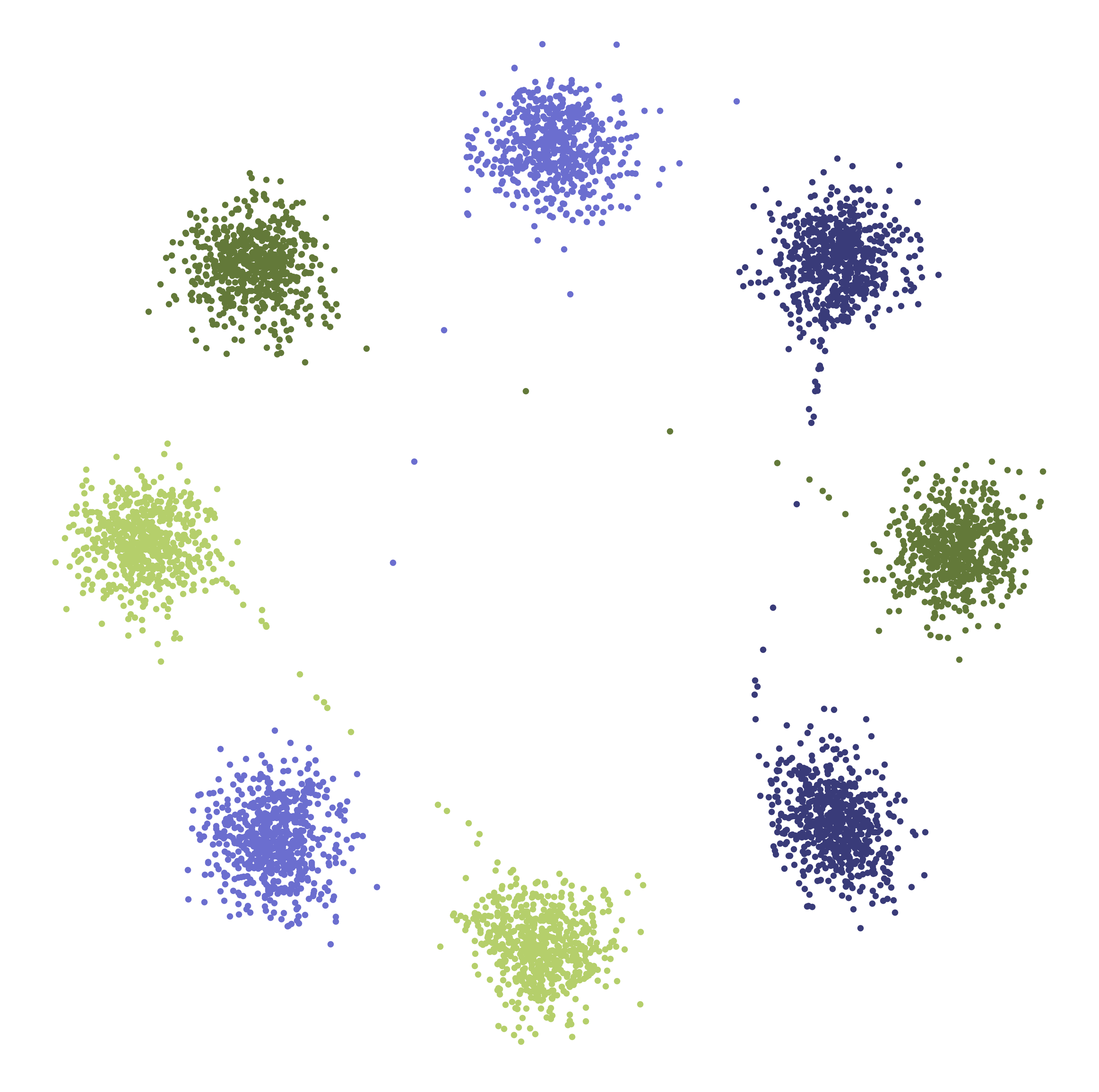}}
    \subfigure[$n=8$]{\includegraphics[width=0.16\linewidth]{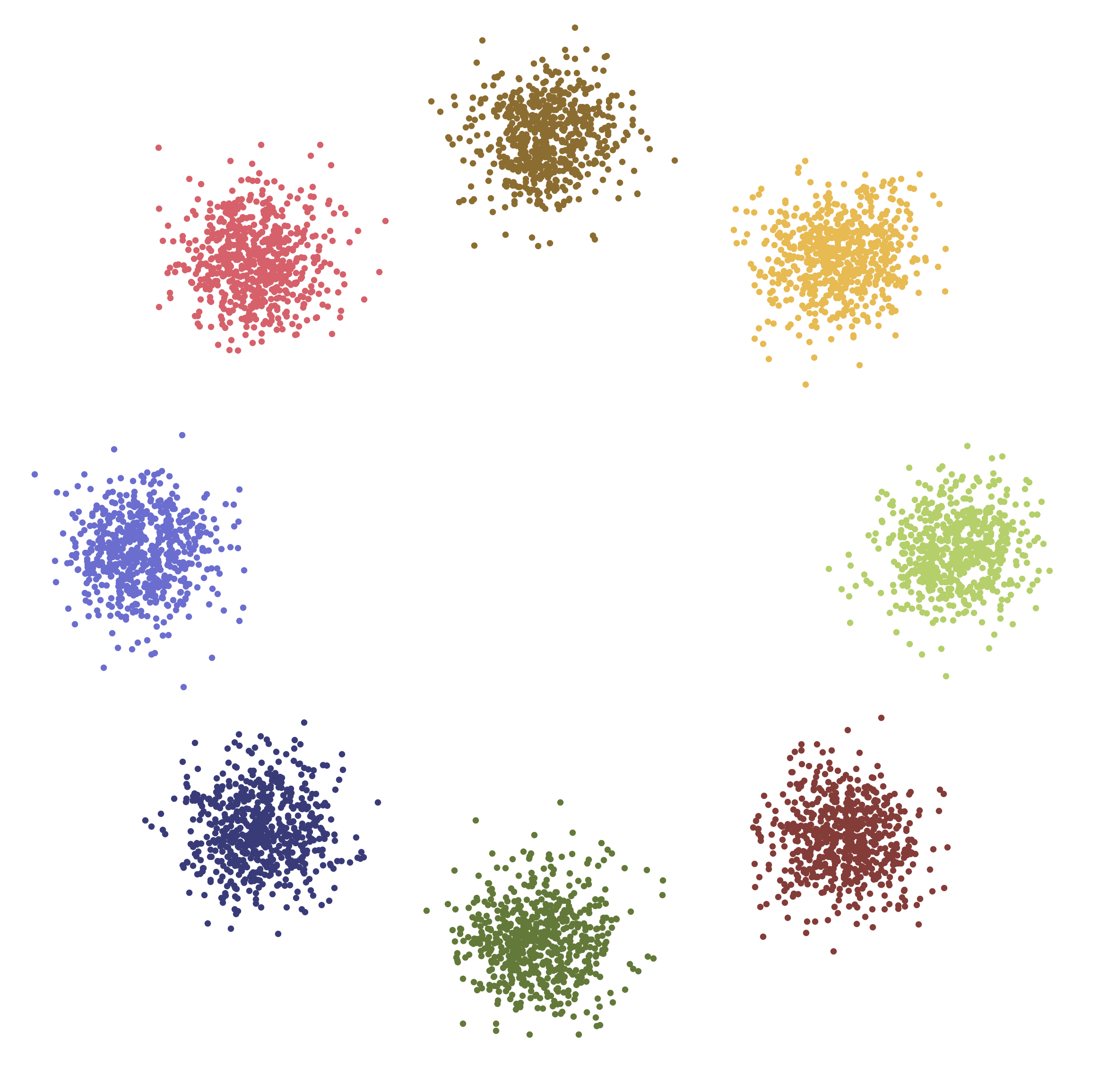}}
    \subfigure[$n=16$]{\includegraphics[width=0.16\linewidth]{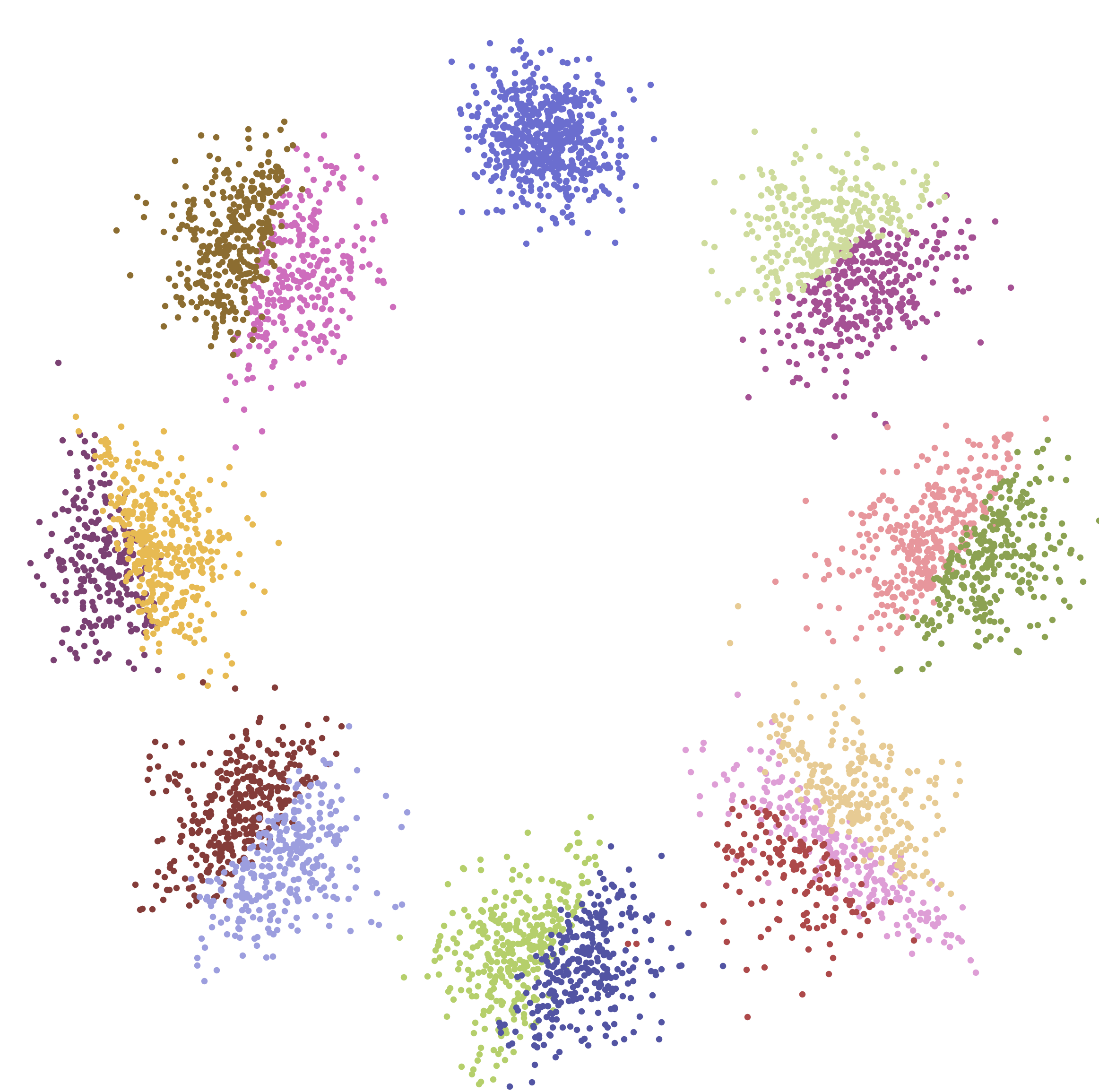}\label{fig:nf16}}
    \caption{Samples generated by RealNVP trained on a mixture of 2D Gaussians. As the number of experts increases, the model generates fewer out-of-distribution samples. Each expert is represented by a unique color.}
\end{figure*}
\begin{table*}[t!]
    \begin{center}
        \small
        \begin{tabular*}{\linewidth}{@{\extracolsep{\fill}}c|cccc|cc}
            \hline
            Metrics &\multicolumn{4}{c|}{\textbf{CIFAR-10}} & \multicolumn{2}{c}{\textbf{CelebA}} \\
            $n$ & $1$ & $2$ & $4$ & $2^*$ & $1$ & $2$   \\
            \hline
            FID $(\downarrow)$ &  $8.21 \pm 0.72$ & $\mathbf{7.34 \pm 0.03}$ & $7.74 \pm 0.15$ & $8.31 \pm 0.14$ & $5.19 \pm 0.09$ & $\mathbf{4.75 \pm 0.00}$\\
            Precision $(\uparrow)$ & $\mathbf{0.85 \pm 0.00}$ & $\mathbf{0.85 \pm 0.00}$& $0.84 \pm 0.00$ & $\mathbf{0.85 \pm 0.00}$ & $\mathbf{0.78 \pm 0.01}$ &$\mathbf{0.78 \pm 0.00}$ \\
            Recall $(\uparrow)$ & $0.69 \pm 0.02$ & $\mathbf{0.72 \pm 0.01}$ & $\mathbf{0.72 \pm 0.01}$ & $0.69 \pm 0.00$ & $0.62 \pm 0.01$ & $\mathbf{0.63 \pm 0.00}$ \\
            Entropy $(\uparrow)$ & $0$ &$0.981\pm0.018$ &  $0.974\pm0.010$& $1$ &$0$ & $0.997\pm 0.001$ \\
            \hline
        \end{tabular*}\label{tab:gan}
    \end{center}
    \caption{Results of CCLVQ on CIFAR-10 and CelebA datasets. The results show that CCLVQ improves the diversity of the generated samples while maintaining the quality of the samples. We denote by $2^*$ the model with independently trained experts used in a uniform mixture. In bold, the best results. We report the mean and standard deviation over $2$ runs. The FID is lower when it is better, while the Precision and Recall are higher when they are better. The normalized entropy is close to $1$ when the repartition of the experts is uniform.}
\end{table*}
\subsection{Improving Quality of Normalizing Flows on synthetic data}
Normalizing Flows (NFs) are a class of generative models that learn a bijective mapping between a simple distribution and a complex target distribution. The networks parameterizing this mapping are designed to ensure that the Jacobian determinant of the transformation remains tractable, allowing for the computation of sample likelihoods. However, due to the invertibility constraint, NFs often struggle to capture multimodal distributions \cite{pmlr-v119-cornish20a,pmlr-v189-verine23a}. While these models are effective at producing diverse outputs, they tend to generate out-of-domain samples.
To enhance the expressivity of the model, we propose using CCLVQ. We employ the RealNVP architecture \cite{dinh2017density} and train it on a mixture of 2D Gaussians. The model is trained with \( n = 1, 2, 4, 8, \) and \( 16 \) experts, where only the parameters corresponding to the highest-likelihood samples are updated.
Figure~\ref{fig:nfdata} illustrates the dataset samples. The results, presented in Figures~\ref{fig:nf1} to \ref{fig:nf16}, provide insights into how the experts are distributed across the dataspace.

\subsection{Improving Diversity Generative Adversarial Networks on CIFAR-10 and CelebA}
\paragraph{Adapting CCLVQ to GANs:} Generative Adversarial Networks are known to suffer from limited diversity in generated samples \cite{verine2023precisionrecall}. We propose to use CCLVQ to improve the diversity of GANs by learning a multimodal distribution. GANs are typically composed of two neural networks: a generator $G:\R^k\rightarrow\R^d$ mapping Gaussian samples to data-like samples and a discriminator $D:\R^d\rightarrow[0,1]$ used to estimate a divergence between a target distribution and the distribution generated by $G$. We use the framework introduced by \citet{brock2018large} where $D$ is trained based on samples from the dataset and from the generator using hinge loss, and $G$ is trained by minimizing $\ell (G(x))\coloneqq -D(G(x))$.  Here, we propose to replace the generator by a mixture of experts $G_i$. The discriminator is now trained on samples from the dataset and from each generator. However, for a given latent sample $X$, only the generator $G_i$ that minimizes $\ell(G_i(X))$ is trained.

\textbf{Training and Evaluating models} We train BigGAN models \cite{brock2018large} on CIFAR-10 and CelebA datasets. For CIFAR-10, we train a BigGAN model with $n=1$, $n=2$, and $n=4$ with two different seeds. As a baseline, we train $2$ models independently and use the uniform mixture of the two models as a comparison. For CelebA, we train a BigGAN model with $n=1$ and $n=2$ with two different seeds. To evaluate the improvements brought by CCLVQ, we use Frechet Inception Distance (FID)~\cite{heusel_gans_2017}, traditionally used to evaluate proximity to the true distribution. We also use the metrics of Precision and Recall \cite{kynkaanniemi_improved_2019} to assess independently the quality of the generated samples and their diversity. To evaluate how the experts are used, we compute the entropy of the classifier's weights normalized by the entropy of the uniform distribution. We report the mean and standard deviation over $2$ runs. The results are presented in Table~\ref{tab:gan}.  Generated samples are shown in Figures~\ref{fig:biggancifar} and \ref{fig:bigganceleba} in Appendix~\ref{appendix:biggan}.

We observe that CCLVQ improves (1) the overall performance of the model (lower FID), and (2) the diversity of the generated samples (higher Recall) with fixed Precision. We can also observe that the normalized entropy is close to $1$ indicating that the experts are used uniformly. The experiments also highlight the improvement brought by CCLVQ since the model with independently trained experts performs as well as the model with $n=1$.

\section{Conclusion and Discussion}
We introduced Conditional Competitive Learning Vector Quantization (CCLVQ) to approximate conditional distributions in this work. Our approach extends traditional vector quantization methods by incorporating learnable functions that adapt to the input distribution, enabling effective uncertainty quantification and diverse and precise sample generation. 
We demonstrated the theoretical grounding of our method and its practical applications across multiple domains, including image inpainting, denoising, and generative modeling. Experimental results validate the efficacy of CCLVQ in capturing multimodal structures while maintaining computational efficiency.

With a finite number of examples \((X_i, Y_i)\), \(1 \leq i \leq N\), the algorithm, at best, provides an approximation of \(\Loi(\tilde Y \mid \tilde X)\), where \((\tilde X, \tilde Y)\) follows the empirical distribution. If \(N\) is sufficiently large, this approximation can be accurate; however, the size of the neural networks also plays a significant role in this context. A detailed study of this problem is beyond the scope of this paper.

% Our method can be complemented with PCA to analyze variability around the conditional mean of individual experts for refined uncertainty quantification.

\clearpage
\newpage

\section*{Impact Statement}
This paper presents work whose goal is to advance the field of 
Machine Learning. There are many potential societal consequences 
of our work, none of which we feel must be specifically highlighted here.

\bibliography{main}

\bibliographystyle{plainnat}

\clearpage
\newpage

\appendix
\onecolumn

\section{Proof}
\label{appendix:proofs}

\begin{theorem*}
    There exists a quantizer $f^*:E \to (\R^d)^n$ such that
    $$ \Delta_n(f^*)=\inf_f \Delta_n(f). $$
    Moreover, for any optimal quantizer $f^*$, with probability $1$, one has
    \[
        \mathcal{W}_2\bigl(Q(X, \cdot), \widehat Q_{f^*}(X, \cdot)\bigr) = \min_{\nu \in \mathcal{P}_n(\mathbb{R}^d)} \mathcal{W}_2\bigl(Q(X, \cdot), \nu\bigr),
    \]
    where \(\mathcal{P}_n(\mathbb{R}^d)\) denotes the set of distributions on \(\mathbb{R}^d\) with support of cardinality at most \(n\).
\end{theorem*}

\begin{proof}
    We construct a quantizer $f^*:E\to(\R^d)^n$ as follows: for each $x\in E$ such that
    $\int_{\R^d} \abs{y}^2 Q(x,dy)<\infty$ choose $f^*(x)=(f^*_1(x), \dots, f^*_n(x))\in(\R^d)^n$
    satisfying
    \begin{align*}
         & \int_{\mathbb{R}^d} \min_{1 \leq i \leq n} \bigl|y - f_i^*(x)\bigr|^2 Q(x, dy)                                       \\
         & = \min_{\alpha \in (\mathbb{R}^d)^n} \int_{\mathbb{R}^d} \min_{1 \leq i \leq n} \bigl|y - \alpha_i\bigr|^2 Q(x, dy).
    \end{align*}
    Such a $f^*(x)$ exists since we know that the distribution $Q(x,\cdot)$ admits at least one optimal quantizer. As $\int_{\R^d} \abs{y}^2 Q(X,dy) < \infty$ with probability $1$, $f^*(X)$ is well defined with probability $1$.
    Let us prove that $\Delta_n(f^*) \le \Delta_n(f)$ for all $f$. On the one hand, for any $f$, one has
    \begin{align*}
        \Delta_n(f^*) & = \mathbb{E}\Bigl[\min_{1 \leq i \leq n} \bigl|Y - f_i(X)\bigr|^2\Bigr]                     \\
                      & = \mathbb{E}\Bigl[ \E\Bigpar{ \min_{1 \leq i \leq n} \bigl|Y - f_i(X)\bigr|^2\mid X} \Bigr]
    \end{align*}
    On the other hand, for any $f$, with probability $1$, one has
    \begin{align*}
         & \E\Bigpar{ \min_{1 \leq i \leq n} \bigl|Y - f^*_i(X)\bigr|^2\mid X}     \\
         & = \int_{\R^d} \min_{1 \leq i \leq n} \bigl|y - f^*_i(X)\bigr|^2 Q(X,dy) \\
         & \le \int_{\R^d} \min_{1 \leq i \leq n} \bigl|y - f_i(X)\bigr|^2 Q(X,dy) \\
         & = \E\Bigpar{ \min_{1 \leq i \leq n} \bigl|Y - f_i(X)\bigr|^2\mid X}
    \end{align*}
    By taking the expectation of both side, one obtains that $\Delta_n(f^*) \le \Delta_n(f)$.

    Moreover if $f$ is optimal, that is $\Delta_n(f)=\Delta_n(f^*)$, one has necessarily, with probability $1$:
    \begin{align*}
         & \int_{\R^d} \min_{1 \leq i \leq n} \bigl|y - f_i(X)\bigr|^2 Q(X,dy)                                                  \\
         & = \int_{\R^d} \min_{1 \leq i \leq n} \bigl|y - f^*_i(X)\bigr|^2 Q(X,dy)                                              \\
         & = \min_{\alpha \in (\mathbb{R}^d)^n} \int_{\mathbb{R}^d} \min_{1 \leq i \leq n} \bigl|y - \alpha_i\bigr|^2 Q(X, dy).
    \end{align*}
    In other words, with probability $1$, $f(X)$ is an optimal quantizer for $Q(X,\cdot)$ and Proposition \ref{simple2} yields:
    \[
        \mathcal{W}_2\bigl(Q(X, \cdot), \widehat Q_{f}(X, \cdot)\bigr) = \min_{\nu \in \mathcal{P}_n(\mathbb{R}^d)} \mathcal{W}_2\bigl(Q(X, \cdot), \nu\bigr).
    \]
\end{proof}

% \begin{proposition}
%     For all $\alpha\in(\R^d)^n$, if
%     $$G_n(\alpha) = \E\Bigcro{\Bigpar{1_{\{I_\alpha(Y)=i\}} (\alpha_i-Y)}_{1\le i\le n}}$$
%     one has
%     $$ D_n\bigpar{\alpha-\epsilon G_n(\alpha)} \le
%         D_n(\alpha) -(2\epsilon-\epsilon^2)\sum_{i=1}^n \abs{G_n(\alpha)_i}^2$$
% \end{proposition}

% \begin{proof}
%     \begin{align*}
%         D_n\bigpar{\alpha-\epsilon G_n(\alpha)} & = \E\bigcro{\min_{1\le i\le n} \abs{\alpha_i-\epsilon G_n(\alpha)_i-Y}^2}        \\
%                                                 & \le \E\bigcro{\abs{\alpha_{I_\alpha(Y)}-\epsilon G_n(\alpha)_{I_\alpha(Y)}-Y}^2} \\
%                                                 & =
%     \end{align*}
% \end{proof}

\section{Experimental details}

\subsection{Inpainter Architecture and Training Details on MNIST}
\label{appendix:traning_detail_mnist}
On the Experiment~\ref{experiment:inpainting_mnist}, the Inpainter model employs an encoder-bottleneck-decoder~\cite{pathak2016context} structure for image inpainting. The encoder compresses input images using two convolutional layers with Batch Normalization and ReLU. The bottleneck captures abstract representations with two convolutional layers. The decoder reconstructs the image with transposed convolutional layers and outputs through a Sigmoid-activated layer.
Training is conducted on MNIST with upper-image masking, using a batch size of $128$ for $200$ epochs. The optimizer is Adam with a $0.001$ learning rate, and the loss function is Mean Squared Error (MSE). Ten expert models are trained to handle different reconstruction modes, with periodic addition of new experts. A noise loss regularization term is applied, and the denoising mode ensures robustness to noisy inputs.

\subsection{Comparison to one single model for reconstruction}

\begin{figure}[t]
    \centering
    \includegraphics[width=0.2\linewidth]{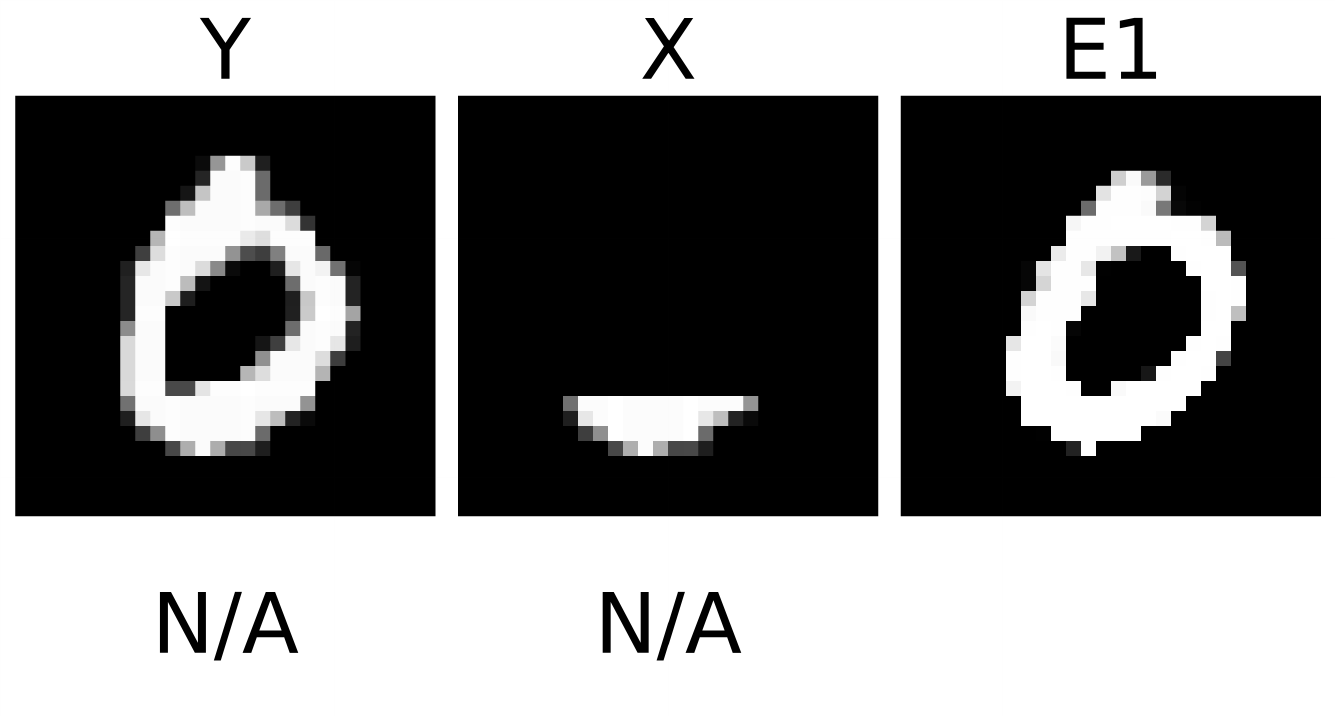}
    
    \includegraphics[width=0.2\linewidth]{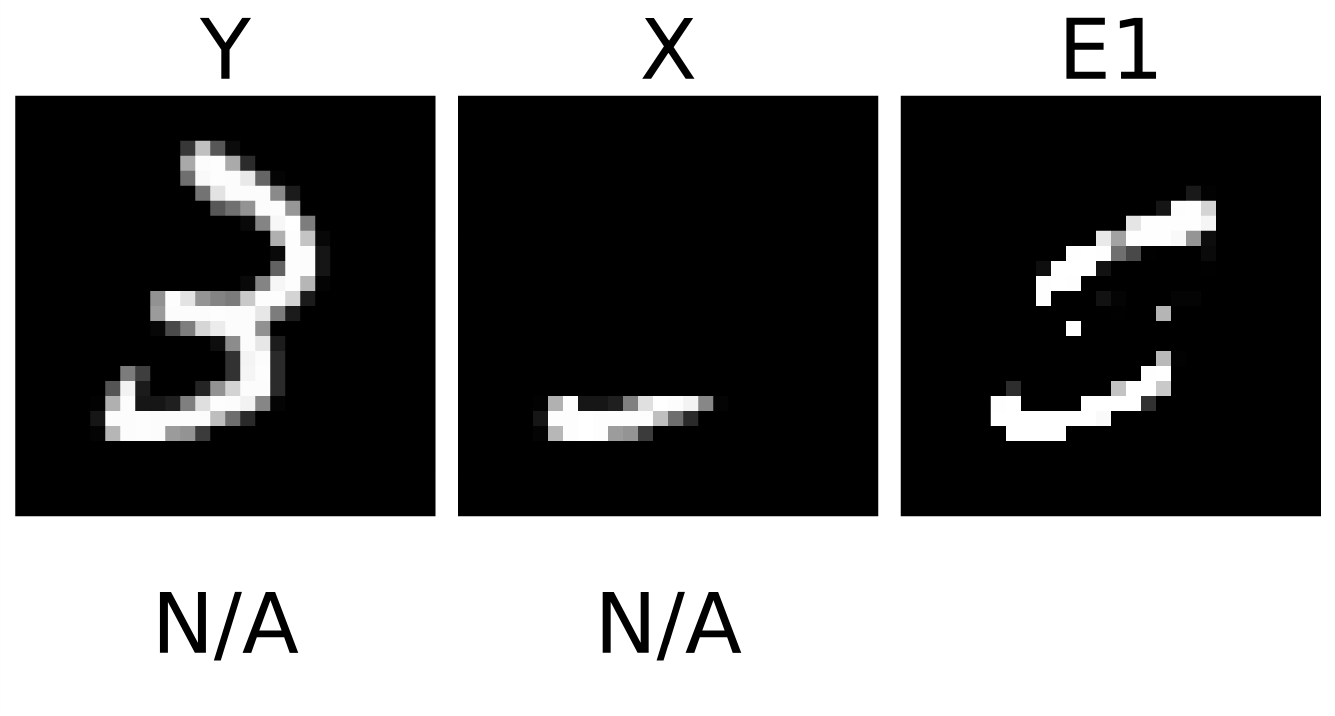}
    
    \includegraphics[width=0.2\linewidth]{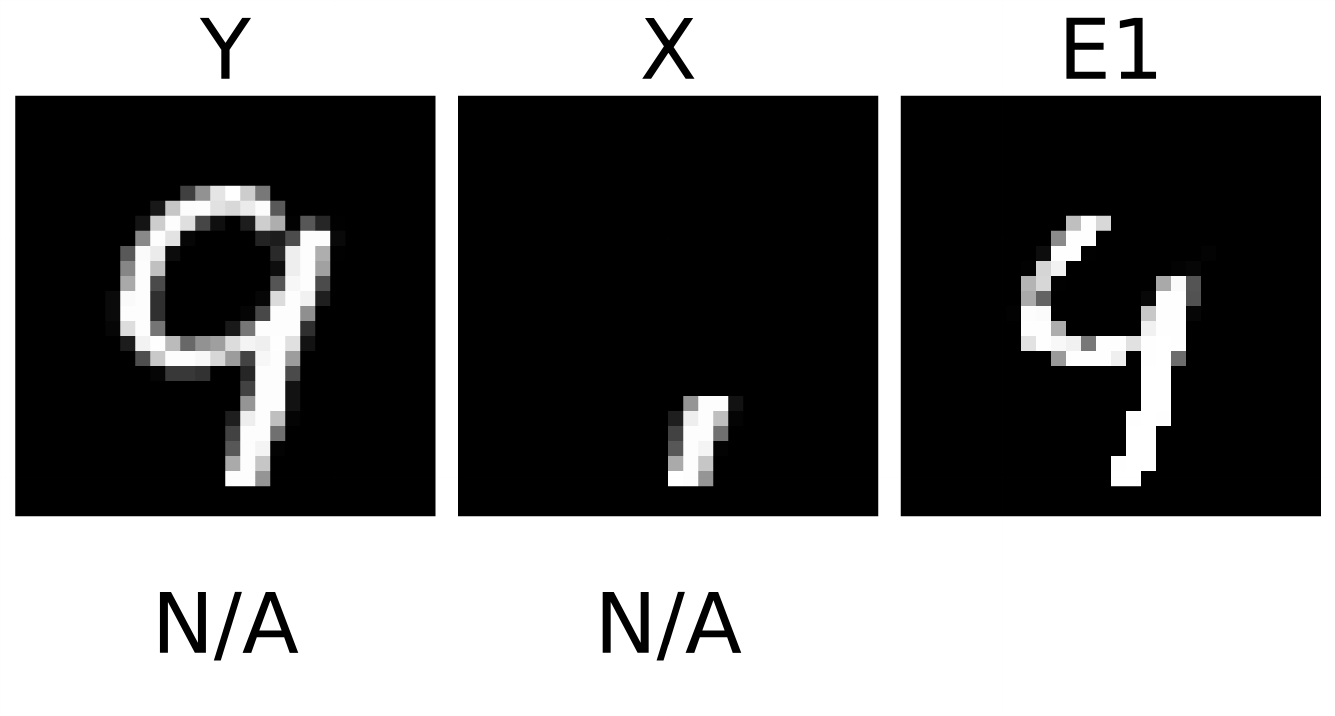}
    
    \caption{Inpainting results on MNIST for a unique model case. Each row shows the ground truth $Y$, masked input $X$, and reconstructions from the model.}
    \label{fig:one_expert_inpainting}
\end{figure}

\label{experiment:inpainting_mnist}
\begin{figure}[t]
    \centering
    \includegraphics[width=0.2\linewidth]{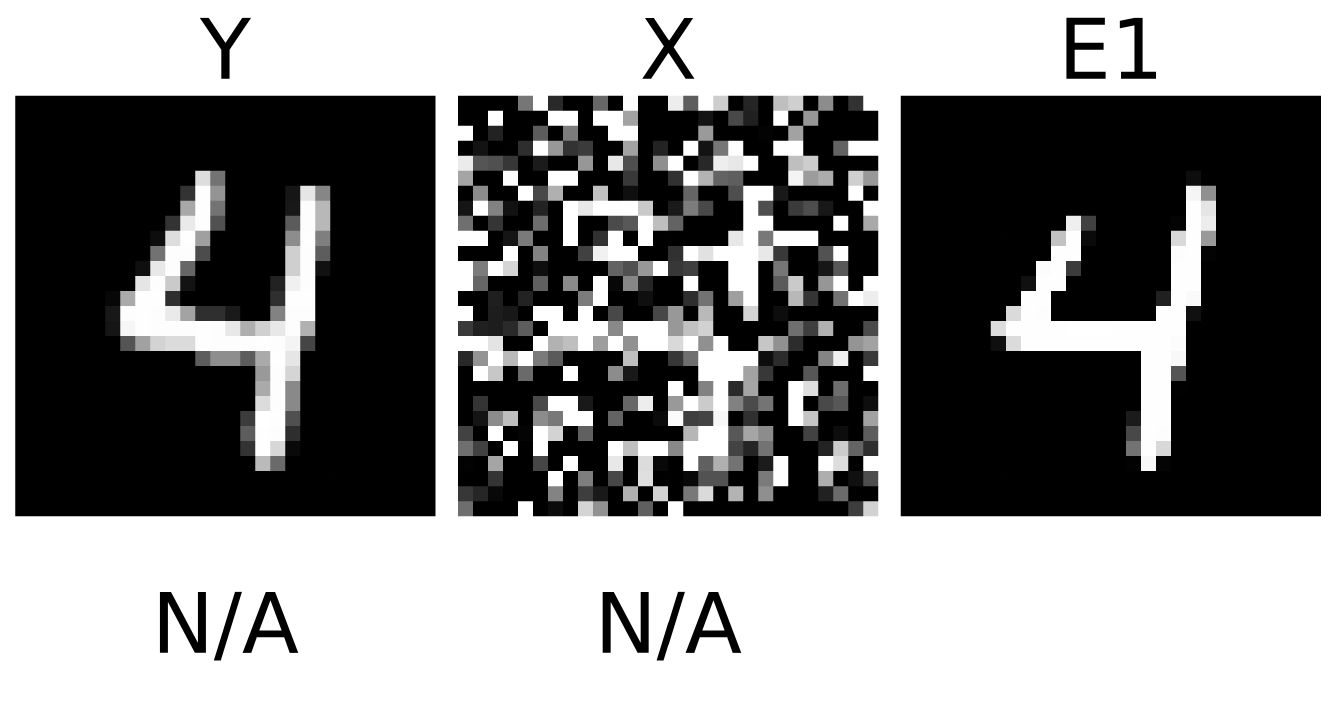}
    
    \includegraphics[width=0.2\linewidth]{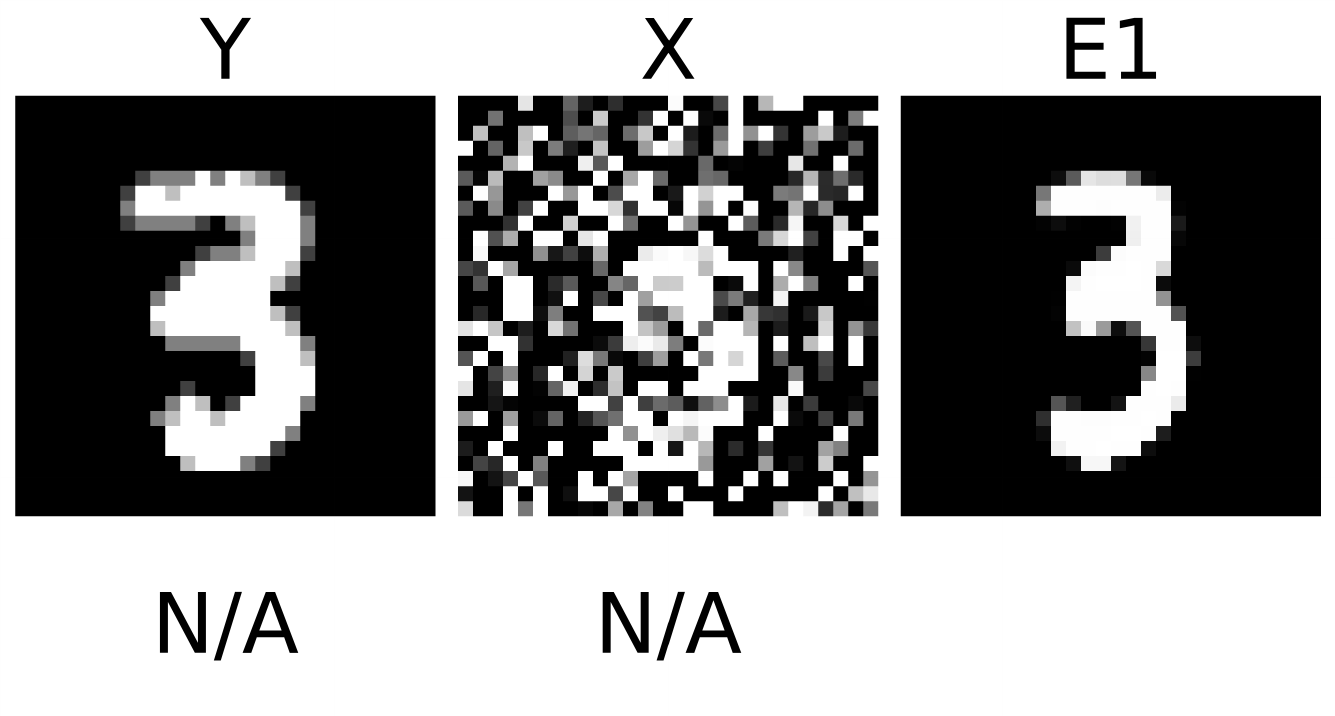}
    
    \includegraphics[width=0.2\linewidth]{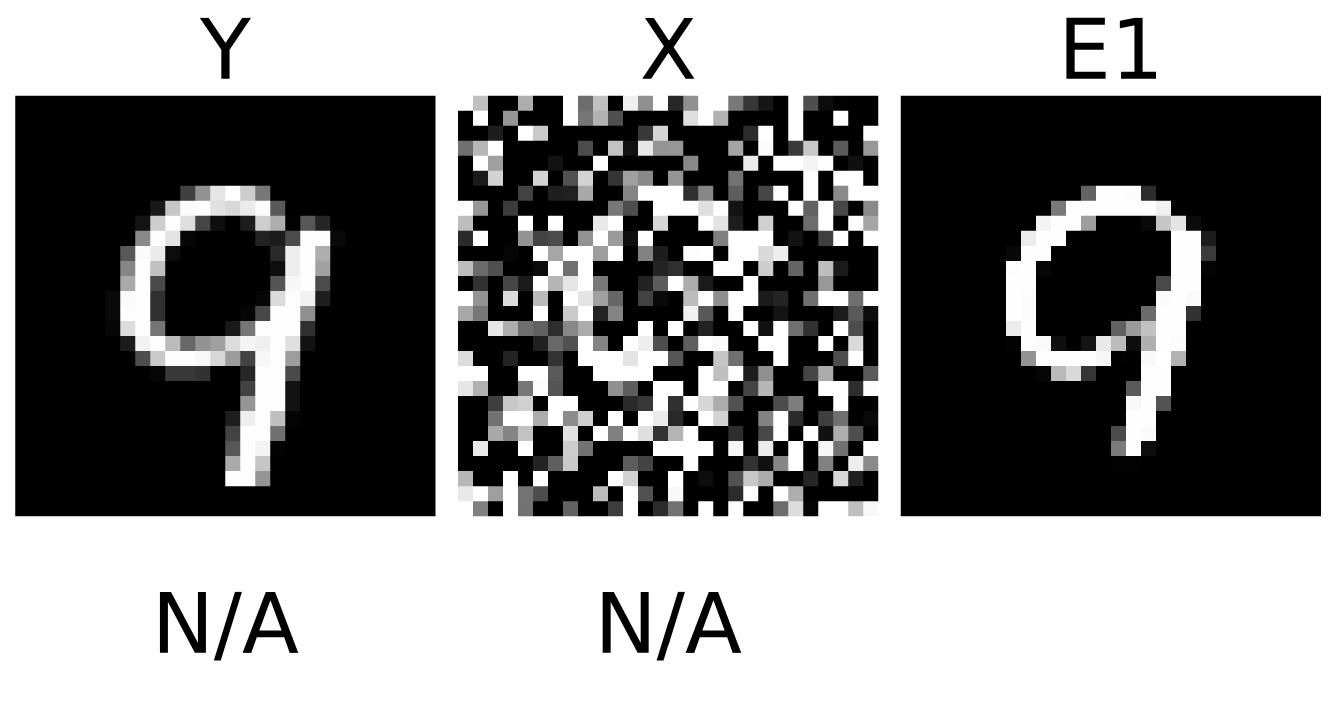}

    \caption{Denoising results on MNIST for a unique model case. Each row shows the ground truth $Y$, noised input $X$, and reconstructions from the model.}
    \label{fig:one_expert_denoising}
\end{figure}

\subsection{Training details on CelebA-HQ}
\label{appendix:traning_detail_celebA_hq}
The training process for the CelebA-HQ dataset began with a baseline model trained for 30 epochs. Starting from this baseline, $9$ additional networks were duplicated, resulting in 10 experts. These experts were trained for one additional epoch using the CCLVQ training algorithm.
The initial assignment of samples to experts was guided by a noisy loss function (small Gaussian noise added to the loss), designed to encourage each expert to specialize in a distinct subset of the data and to avoid a winner-takes-all effect.
Batch accumulation stabilized training by effectively increasing the batch size by a factor of 8 (from 16 to 128). This ensured smoother gradient updates and enhanced convergence during the training process.
The training was performed on 1 GPU A100.

% \begin{figure}
%     \centering
%     \includegraphics[width=1\linewidth]{images_celeb_hd/sample_1.pdf}
%     \caption{Diverse inpainting results on the CelebA-HQ dataset using CCLVQ. Each row shows reconstructions from different expert models after masking the area near the eyes.
%         These results illustrate the ability of CCLVQ to model multimodal distributions effectively.}
%     \label{fig:more_faces}
% \end{figure}
\begin{figure}[t]
    \centering
    \includegraphics[width=0.9\linewidth]{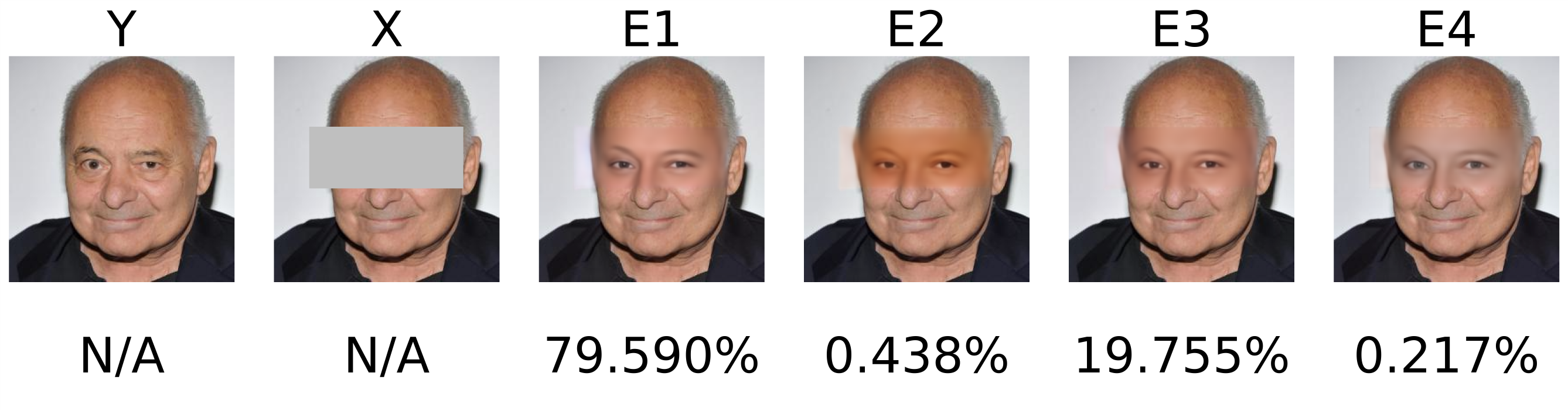}
    \includegraphics[width=0.9\linewidth]{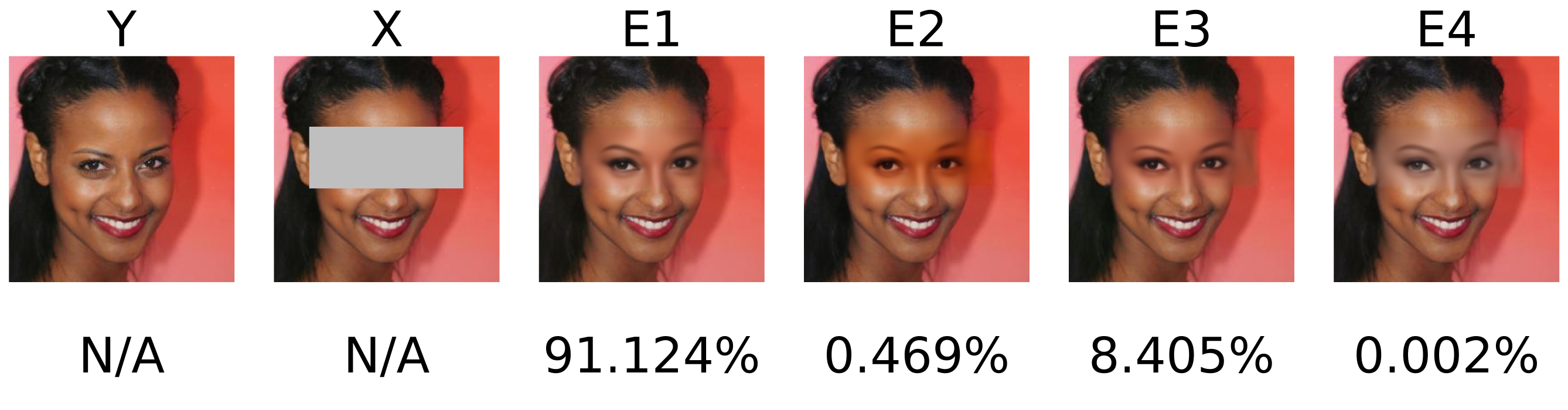}
    \includegraphics[width=0.9\linewidth]{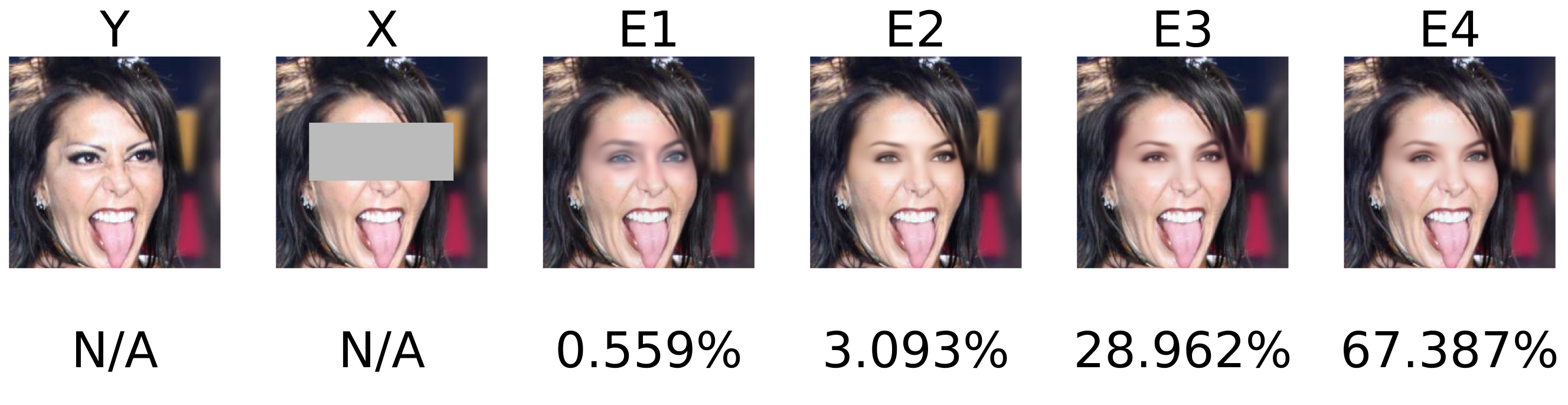}
    \centering
    \includegraphics[width=0.9\linewidth]{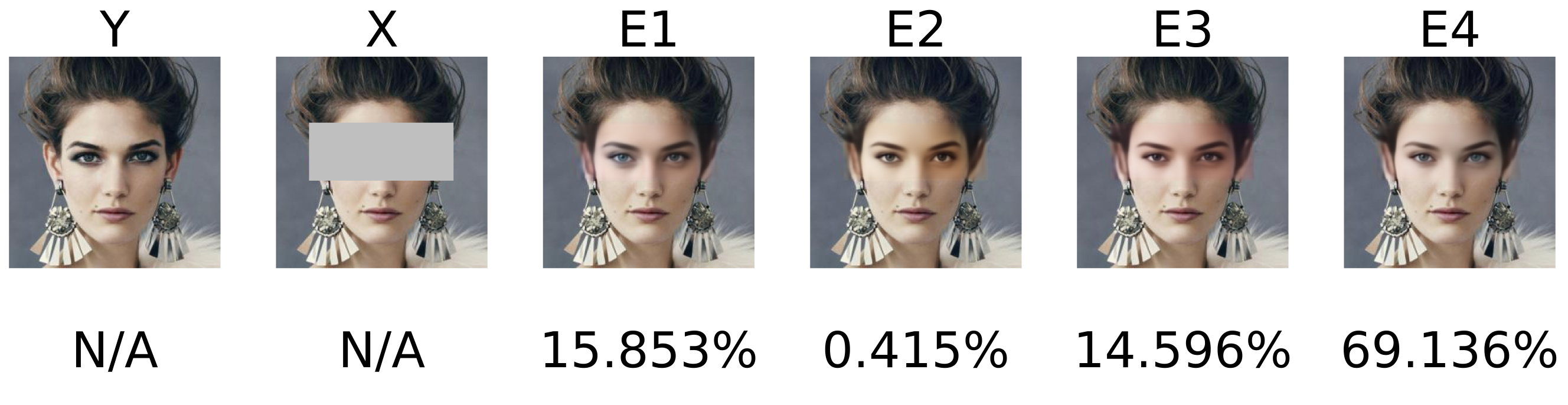}
    \centering
    \includegraphics[width=0.9\linewidth]{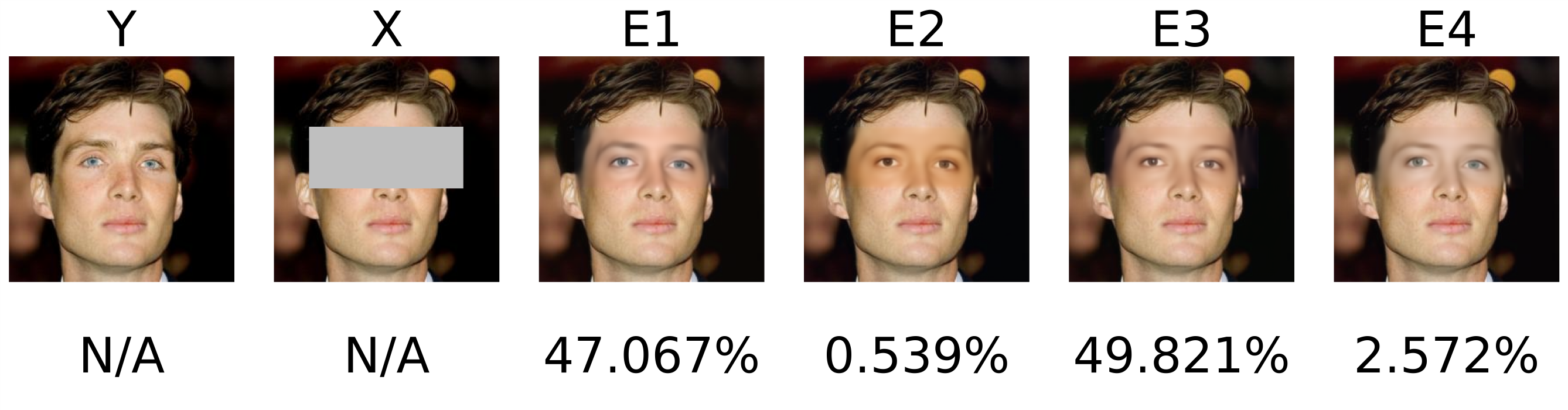}
    \includegraphics[width=0.9\linewidth]{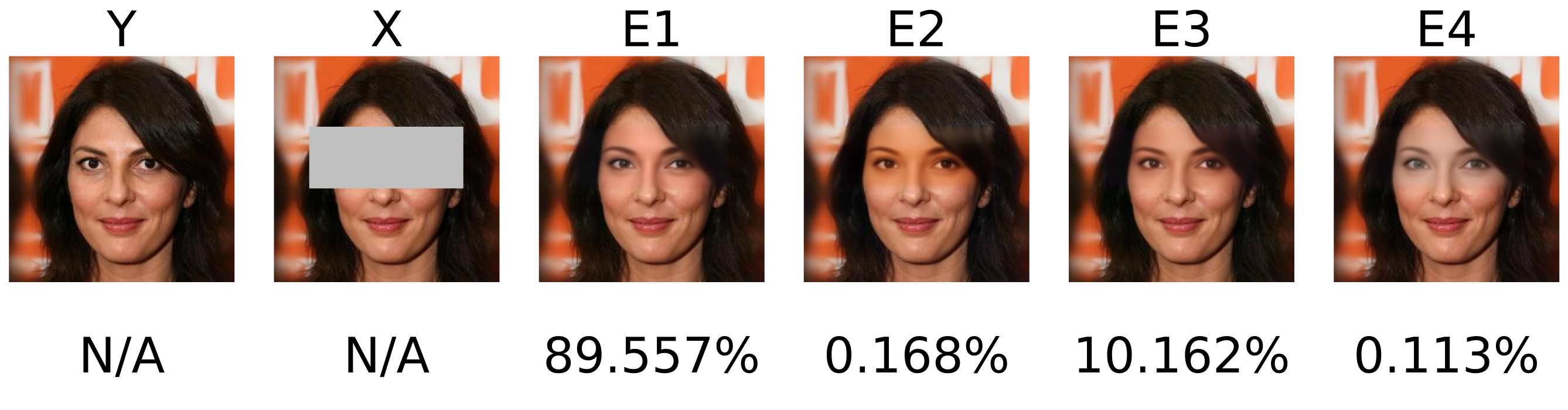}
    \caption{More inpainting results on the CelebA-HQ dataset using CCLVQ. 
    }
    \label{fig:more_celeb_inpainting}
\end{figure}

\subsection{Training details on BigGAN}\label{appendix:biggan}
The training process for BigGAN models on CIFAR-10 and CelebA datasets is as follows. We use the PyTorch implementation given in \url{https://github.com/ajbrock/BigGAN-PyTorch}. We use the same set of hyperparameters given in the repository. For multiple experts, we stack the number of gradient updates until all experts receive the same number of updates as the baseline model.  For CIFAR-10, the model has been trained on 2 GPUs V100 16Go for 500 epochs with a batch size of 50. The learning rate is set to $0.0002$ and the Adam optimizer is used. For CelebA, the model has been trained on 4 GPUs V100 32Go for 100 epochs with a batch size of 128. The learning rate is set to $0.00002$ and the Adam optimizer is used. We evaluate the models using the FID, Precision, and Recall metrics with 50k samples and $k=3$. We report some visualization of the generated data in Figures~\ref{fig:biggancifar} and \ref{fig:bigganceleba}.

\begin{figure}
    \subfigure[$n=1$]{\includegraphics[width=0.24\linewidth]{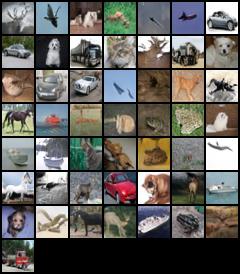}}
    \subfigure[$n=2$, Model 1]{\includegraphics[width=0.24\linewidth]{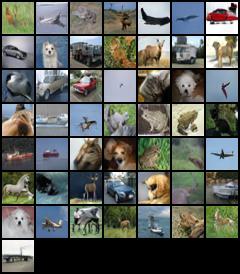}}
    \subfigure[$n=2$, Model 2]{\includegraphics[width=0.24\linewidth]{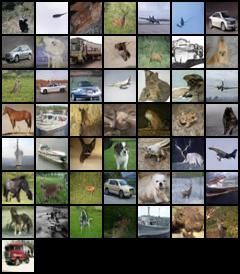}}
    \subfigure[$n=2$, Mixture]{\includegraphics[width=0.24\linewidth]{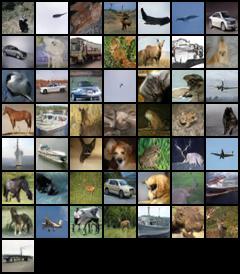}}
    \caption{Generated CIFAR-10 samples from BigGAN models with $n=1$ and $n=2$ experts.}
    \label{fig:biggancifar}
\end{figure}

\begin{figure}
    \subfigure[$n=1$]{\includegraphics[width=0.24\linewidth]{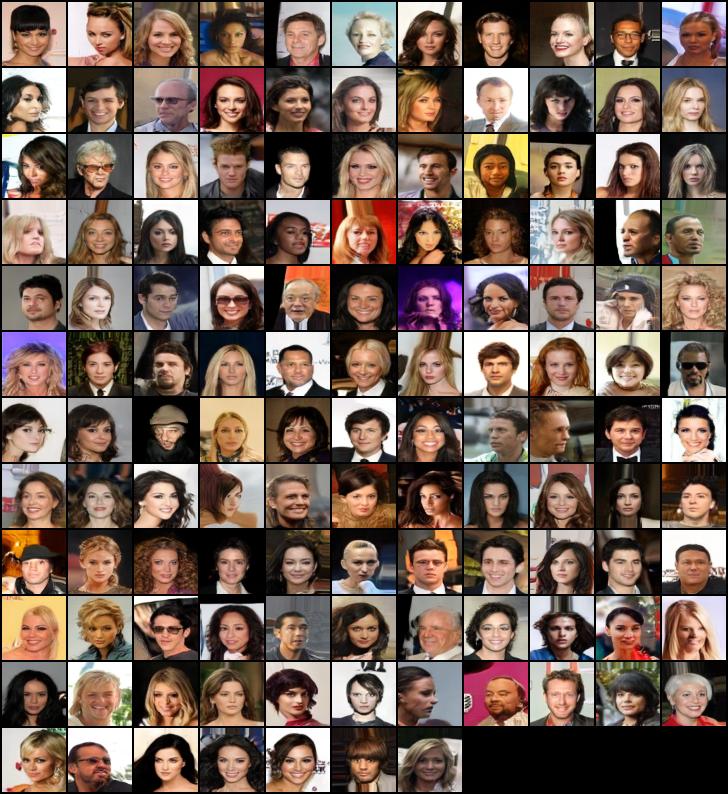}}
    \subfigure[$n=2$, Model 1]{\includegraphics[width=0.24\linewidth]{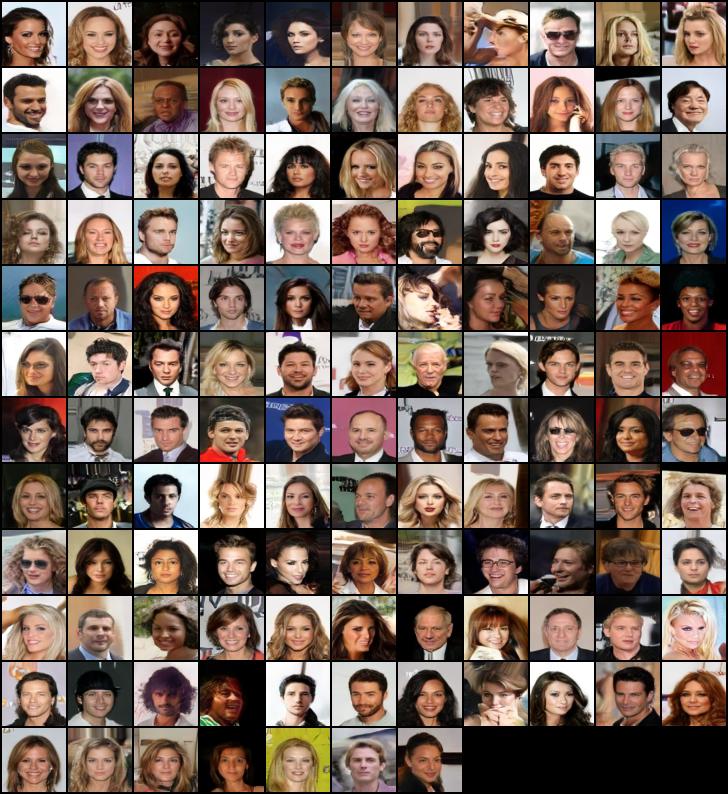}}
    \subfigure[$n=2$, Model 2]{\includegraphics[width=0.24\linewidth]{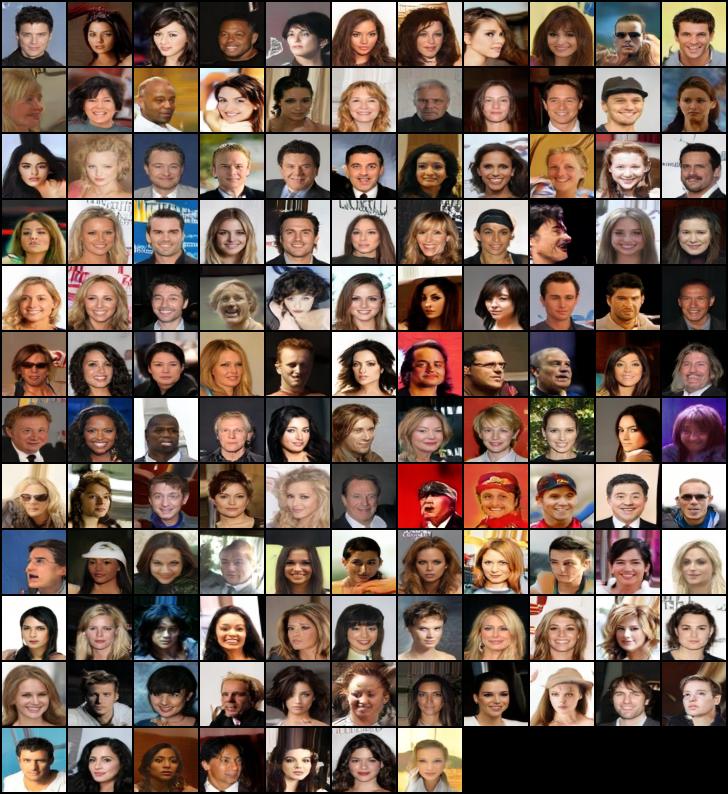}}
    \subfigure[$n=2$, Mixture]{\includegraphics[width=0.24\linewidth]{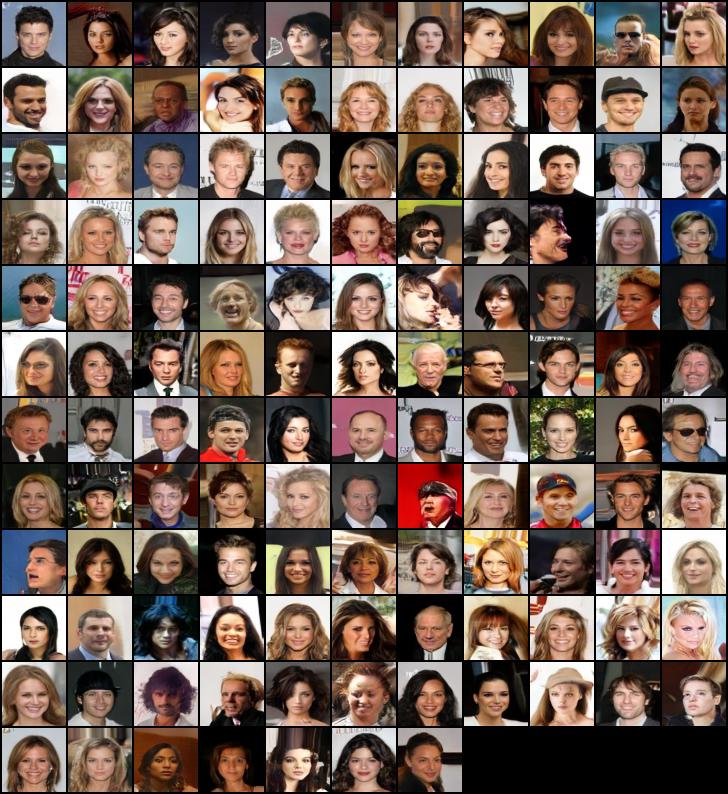}}
    \caption{Generated CelebA samples from BigGAN models with $n=1$ and $n=2$ experts.}
    \label{fig:bigganceleba}
\end{figure}

\end{document}